\documentclass[twoside,11pt]{article}

\usepackage{jmlr2e}

\usepackage{amsmath}
\usepackage{enumerate}
\usepackage{mathtools}
\usepackage{setspace}
\usepackage{url} 
\usepackage{appendix}
\usepackage[ruled,vlined]{algorithm2e}
\usepackage{stix2}
\usepackage{hyperref}
\usepackage{xcolor}

\DeclareMathOperator*{\argmax}{argmax}

\DeclareMathOperator{\E}{\mathbb E}
\DeclareMathOperator{\Var}{\mathrm{Var}}

\DeclareMathOperator{\supp}{\mathrm supp}
\newcommand{\Real}{\mathbb R}

\newcommand{\NatInt}{\mathbb N}
\newcommand{\BFx}{\mathbf x}

\newcommand{\BFZ}{\mathbf Z}
\newcommand{\BFX}{\mathbf X}
\newcommand{\BFK}{\mathbf K}
\newcommand{\BFY}{\mathbf Y}
\newcommand{\BFA}{\mathbf A}
\newcommand{\BFa}{\mathbf a}
\newcommand{\BFl}{\mathbf l}
\newcommand{\BFphi}{\boldsymbol{\phi}}
\newcommand{\CalO}{\mathcal O}
\newcommand{\CalK}{\mathcal K}
\newcommand{\CalI}{\mathcal I}
\newcommand{\CalF}{\mathcal F}
\newcommand{\BFbeta}{\boldsymbol{\beta}}
\newcommand{\BFomega}{\boldsymbol{\omega}}

\allowdisplaybreaks

\ShortHeadings{Kernel Packet}{Chen, Ding and Tuo}
\firstpageno{1}

\begin{document}

\title{\bf Kernel Packet: An Exact and Scalable Algorithm for Gaussian Process Regression with Mat\'ern Correlations}
  
 \author{\name Haoyuan Chen \email chenhaoyuan2018@tamu.edu\\
    \name Liang Ding \email ldingaa@tamu.edu\\
    \name Rui Tuo\thanks{First two authors contributed equally to this work.
    }\hspace{.2cm} \email ruituo@tamu.edu\\
    \addr Wm Michael Barnes '64 Department of Industrial \& Systems Engineering\\
    Texas A\&M University\\
    College Station, TX 77843, USA
    }

\editor{ }

\maketitle

\begin{abstract}
    We develop an exact and scalable algorithm for one-dimensional Gaussian process regression with Mat\'ern correlations whose smoothness parameter $\nu$ is a half-integer. The proposed algorithm only requires $\mathcal{O}(\nu^3 n)$ operations and $\mathcal{O}(\nu n)$ storage. This leads to a linear-cost solver since $\nu$ is chosen to be fixed and usually very small in most applications. The proposed method can be applied to multi-dimensional problems if a full grid or a sparse grid design is used. The proposed method is based on a novel theory for Mat\'ern correlation functions. We find that a suitable rearrangement of these correlation functions can produce a compactly supported function, called a ``kernel packet''. Using a set of kernel packets as basis functions leads to a sparse representation of the covariance matrix that results in the proposed algorithm. Simulation studies show that the proposed algorithm, when applicable, is significantly superior to the existing alternatives in both the computational time and predictive accuracy.
\end{abstract}





\section{Introduction}
\label{sec:intro}

Gaussian process (GP) regression is a powerful function reconstruction tool. It has been widely used in computer experiments \citep{santner2003design,gramacy2020surrogates}, spatial statistics \citep{cressie2015statistics}, supervised learning \citep{Rasmussen06gaussianprocesses}, reinforcement learning \citep{deisenroth2013gaussian}, probabilistic numerics \citep{hennig2015probabilistic} and Bayesian optimization \citep{srinivas2009gaussian}. GP regression models are flexible to fit a variety of functions, and they also enable uncertainty quantification for prediction by providing predictive distributions. With these appealing features, GP regression has become the primary surrogate model for computer experiments since popularized by \cite{sacks1989design}. Despite these advantages, Gaussian process regression has its drawbacks. A major one is its computational complexity. Training a GP model requires furnishing matrix inverses and determinants. With $n$ training points, each of these matrix manipulations takes $\mathcal{O}(n^3)$ operations (referred to as ``time'' thereafter, assuming for simplicity that no parallel computing is enforced) if a direct method, such as the Cholesky decomposition, is applied. Besides, the computation for model training may also be hindered by the $\mathcal{O}(n^2)$ storage requirement \citep{gramacy2020surrogates} to store the $n\times n$ covariance matrix. 

Tremendous efforts have been made in the literature to address the computational challenges of GP regression. Recent advances in scalable GP regression include random Fourier features \citep{rahimi2007random}, Nystr\"om Approximation (also known as inducing points) \citep{smola2000sparse,williams2001using,titsias2009variational,bui2017unifying,Katzfuss_2017,ChenStein17}, structured kernel interpolation \cite{wilson2015kernel}, etc. These methods are based on different types of approximation of GPs, i.e., the efficiency is gained at the cost of its accuracy. In contrast, the main objective of this work is to propose a novel scalable approach that does not need an approximation.

In this work, we focus on the use of GP regression in the context of \textit{computer experiments}. In these studies, the training data are acquired through an experiment, in which the input points can be chosen. Such a  choice is called a \textit{design} of the experiment. It is well known that a suitably chosen design can largely simplify the computation. Here we consider the ``tensor-space'' techniques in terms of using a product correlation function and a full grid or a sparse grid \citep{plumlee2014fast} design. The tensor-space techniques can reduce a multivariate GP regression problem to several univariate problems. It is worth noting that, in some applications besides computer experiments, even if the input sites are not controllable, the data are naturally observed on full grids, e.g., the remote sensing data in geoscience applications \citep{bazi2009gaussian}. In these scenarios, the tensor-space techniques are also applicable.

Having the tensor-space techniques, the final hard nut to crack is the one-dimensional GP regression problem. We assume that the one-dimensional input data are already \textit{ordered} throughout this work. This assumption is reasonable in computer experiment applications since the design points are chosen at our will. In other applications where we do not have ordered data in the first place, it takes only $\mathcal{O}(n \log n)$ time to sort them. 

This work presents a mathematically \textit{exact} algorithm to make conditional inference for one-dimensional GP regression with time and space complexity both linear in $n$. This algorithm is specialized for Mat\'ern correlations with smoothness $\nu$ being a half-integer (see Section \ref{subsec:GPR review} for the definition.) Mat\'ern correlations are commonly used in practice \citep{stein2012interpolation,santner2003design,gramacy2020surrogates}.
In most applications, $\nu$ is chosen to be \textit{small}, e.g., $\nu=1.5$ or $\nu=2.5$, for the sake of a higher model flexibility. The proposed algorithm enjoys the following important features.
\begin{itemize}
    \item Given the hyper-parameters of the GP, the proposed algorithm is mathematically exact, i.e., all numerical error is attributed to the roundoff error given by the machine precision.
    \item There is no restriction for the one-dimensional input points. But if the points are equally spaced, the computational time can be further reduced.
    \item It takes only $\mathcal{O}(\nu^3 n)$ time to compute the matrix inversion and the determinant. For equally spaced designs, this time is further reduced to $\mathcal{O}(\nu^2 n)$.
    \item After the above pre-processing time, it takes only $\mathcal{O}(\nu +\log n)$ or even $\mathcal{O}(\nu)$ time to make a new prediction (i.e., evaluate the conditional mean) at an untried point.
    \item The storage requirement is only $\mathcal{O}(\nu n)$.
\end{itemize}

The remainder of this article is organized as follows. We will review the general idea of GP regression and some existing algorithms in Sections \ref{subsec:GPR review} and \ref{subsec:existing}, respectively. The mathematical theory behind the proposed algorithm is introduced in Section \ref{sec:KP_theory}. In Section \ref{sec:algorithm}, we propose the main algorithm. Numerical studies are given in Section \ref{sec:expe}. In Section \ref{sec:extension}, we briefly discuss some possible extensions of the proposed method. Concluding remarks are made in Section \ref{sec:conc}. Appendices \ref{appenA1} and \ref{appenA2} contain the required mathematical tools and our technical proofs, respectively.

\subsection{A review on GP Regression}\label{subsec:GPR review}
Let $Y(\BFx)$ denote a stationary GP prior on $\Real^d$ with mean function $\mu(\BFx)$, variance $\sigma^2$, and correlation function $K(\BFx,\BFx')$. The correlation function is also referred to as a ``kernel function'' in the language of applied math or machine learning \citep{Rasmussen06gaussianprocesses}. When $d=1$, there are two types of popular correlation functions. The first type is the \textit{Mat\'ern} family \citep{stein2012interpolation}:
\begin{equation}
    \label{eq:Matern-1d}
     K(x,x')=\frac{2^{1-\nu}}{\Gamma(\nu)}\left(\sqrt{2\nu}\frac{\lvert x-x'\rvert}{\omega}\right)^{\nu}K_{\nu}\left(\sqrt{2\nu}\frac{\lvert x-x'\rvert}{\omega}\right),
\end{equation}
for any $x,x'\in\Real$, where $\nu>0$ is the smoothness parameter, $\omega>0$ is the scale and $K_\nu$ is the modified Bessel function of the second kind. The smoothness parameter $\nu$ governs the smoothness of the GP  $Y$ \citep{santner2003design,stein2012interpolation}; the scale parameter $\omega$  determines the spread of the correlation \citep{Rasmussen06gaussianprocesses}. Mat\'ern correlation functions are widely used because of its great flexibility. 
The second type is the \textit{Gaussian} family:
\begin{equation}
    \label{eq:Gaussian-kernel}
    K(x,x')=\exp\left(-\frac{\lvert x-x'\rvert^2}{\omega}\right),
\end{equation}
for any $x,x'\in\Real^d$. A Gaussian kernel function is the limit of a sequence of Mat\'ern kernels with the smoothness parameter tending to infinity. The sample paths generated by GP with Gaussian correlation function are infinitely differentiable.

For multi-dimensional problems, a typical choice of the correlation structure is the \textit{separable} or \textit{product} correlation:
\begin{equation}
    \label{eq:separable-kernel}
     K(\BFx,\BFx')=\prod_{j=1}^dK_j(x_j,x'_j),
\end{equation}
for any $\BFx=(x_1,\ldots,x_d)^T,\BFx'=(x'_1,\ldots,x'_d)^T$, where $K_j$ is a one-dimensional Mat\'ern or Gaussian correlation function for each $j$. This assumption ensures that the GP lives in a tensor space, and is key to the ``tensor-space'' techniques, which reduces the multi-dimensional problems to one-dimensional ones. 

Suppose that we have observed $\BFY= \big(Y(\BFx_1),\cdots,Y(\BFx_n)\big)^T$ on $n$ distinct points $\BFX=\{\BFx_i\}_{i=1}^n$. The aim of \textit{GP regression} is to predict the output at an untried input $\BFx^*$ by computing the distribution of $Y(\BFx^*)$ conditional on $\BFY$, which is a normal distribution with the following conditional mean and variance \citep{santner2003design,banerjee2014hierarchical}:
\begin{align}
          &\E\left[Y(\BFx^*)\big|\BFY\right]= \mu(\BFx^*)+K(\BFx^*,\BFX)\BFK^{-1}\left(\BFY-\boldsymbol{\mu}\right),\label{eq:conditional-mean}\\
          & \Var\left[Y(\BFx^*)\big|\BFY\right]=\sigma^2\bigg(K(\BFx^*,\BFx^*)-K(\BFx^*,\BFX)\BFK^{-1}K(\BFX,\BFx^*)\bigg),\label{eq:conditional-variance}
\end{align}
where $\sigma^2>0$ is the variance, $K(\BFx^*,\BFX)=\left(K(\BFX,\BFx^*)\right)^T=\left(K(\BFx^*,\BFx_1),\cdots,K(\BFx^*,\BFx_n)\right)$, $\BFK=\left[K(\BFx_i,\BFx_s)\right]_{i,s=1}^n$ and $\boldsymbol{\mu}= \left(\mu(\BFx_1),\cdots,\mu(\BFx_n)\right)^T$. 

In GP regression, the mean function $\mu$ is usually parametrized as a linear form $\mu=\sum_{i=1}^p\beta_i f_i$ for some unknown coefficient vector $\boldsymbol{\beta}=\big(\beta_1,\cdots,\beta_p\big)^T$ and known regression functions $f_1,\cdots,f_p$. To improve the predictive performance of GP regression, the coefficient vector $\boldsymbol{\beta}$, variance $\sigma^2$ and scales $\boldsymbol{\omega}=\big(\omega_1,\cdots, \omega_d\big)^T$ associated to each one-dimensional correlation function $k_j$ are usually estimated via maximum likelihood \citep{jones1998}. The log-likelihood function given the data is:
\begin{equation}\label{eq:loglike}
    L(\BFbeta,\sigma^2,\BFomega)=-\frac{1}{2}\left[n\log \sigma^2 + \log \det(\BFK) +\frac{1}{\sigma^2} \big(\BFY-\mathbf{F}\BFbeta\big)^T \BFK^{-1}\big(\BFY-\mathbf{F}\BFbeta\big)\right],
\end{equation}
where $\det(\BFK)$ denotes the determinant of the correlation matrix $\BFK$ and $\mathbf{F}$ is the $n\times p$ matrix whose $(i, s)^{\rm th}$ entry is $f_s(\BFx_i)$. The \textit{Maximum Likelihood Estimator} (MLE) is then defined as the maximimizer of the log-likelihood function: $(\hat{\BFbeta},\hat{\sigma}^2,\hat{\BFomega})=\argmax_{\BFbeta,\sigma^2,\BFomega} L(\BFbeta,\sigma^2,\BFomega)$.

In both GP regression and parameter estimation, the computation can become unstable or even intractable because it involves the pursuit of the inversion and the determinant of the correlation matrix $\BFK$. Each task takes $\mathcal{O}(n^3)$ time if a direct method, such as the Cholesky decomposition, is applied, which is a fundamental computational challenge for GP regression.

\subsection{Comparisons with Existing Methods}\label{subsec:existing}

When applicable, as a specialized algorithm, the proposed method is significantly superior to the existing alternatives. In this section, we compare the proposed method with a few popular existing approaches for large-scale GP regression. It is worth noting that, the fundamental mathematical theory for the proposed method differs from that of any of the existing methods. A summary of the comparisons is presented in Table \ref{fig:comparison}.

\begin{table*}[h]
\centering
\scalebox{.7}{
    \begin{tabular}{ | c| c | c | c |c |c |}
    \hline
    Method & Kernels& Design   & Time  & Storage & Accuracy 
    \\ [2ex] 
    \hline 
    Proposed Method & Mat\'ern-$\nu$, $\nu-1/2\in\NatInt$ &Arbitrary& $\mathcal{O}(\nu^3n)$ &$\mathcal{O}(\nu n)$ & Exact\\ [2ex]
    \hline
    &Stationary& Equally&  &  &Depending on the\\
     Toeplitz Methods& Kernels & Spaced &$\CalO(n\log n)$ &$\CalO(n)$ &number of iterations  \\[1.5ex]
    \hline
    Local Approximate GP &  &&  & &  \\
    (with $m$ nearest neighbors) & Arbitrary &Arbitrary & $\CalO(m^3)$&$\CalO(m^2+n)$ &Unknown\\[1.5ex]
    \hline
    Random Fourier Features& Mat\'ern-$\nu,\ \nu>\frac{1}{2}$ &&  &  &  \\
    (with $m$ random features)& Gaussian& Arbitrary& $\CalO(m^2n)$ & $\CalO(m^2+ mn)$ &{ $\CalO_p(m^{-1/2})$}\\[1.5ex]
    \hline
    Nystr\"om Approximation& Mat\'ern-$\nu,\ \nu>\frac{3}{2}$ &&  &  & Mat\'ern: $\CalO_p(m^{-2\nu-1})$ \\
    (with $m$ inducing points)& Gaussian& Arbitrary& $\CalO(m^2n)$ & $\CalO(m^2+ mn)$ &{Gaussian: $\CalO_p(\exp(-\alpha m\log m))$}\\[1.5ex]
    \hline
    \end{tabular}}
    \caption{Comparisons with existing methods. \label{fig:comparison}}
\label{Tab:1}
\end{table*}

\noindent \textbf{Toeplitz methods}: Toeplitz methods \citep{wood1994simulation} work for stationary GPs with \textit{equally spaced} design points. These methods leverage the Toeplitz structure of the covariance matrices under this setting. To make a prediction in terms of solving (\ref{eq:conditional-mean}) and (\ref{eq:conditional-variance}), there are two approaches. The first is to solve the Toeplitz system exactly, using, for example, the Levinson algorithm \citep{zhang2005time}. This takes $\mathcal{O}(n^2)$ time. A more commonly used approach is based on a conjugate gradient algorithm \citep{atkinson2008introduction} to solve the matrix inversion problems in (\ref{eq:conditional-mean}) and (\ref{eq:conditional-variance}). Each step takes $\mathcal{O}(n \log n)$ time. For the sake of rapid computation, the number of iterations is chosen to be small. But then the method becomes inexact. Moreover, the conjugate gradient algorithm is unable to find the determinant in (\ref{eq:loglike}) \citep{wilson2015kernel}. Thus one has to resort to the exact algorithm to compute the likelihood value, which takes $\mathcal{O}(n^2)$ time. 
Toeplitz methods only works for equally spaced design points. This is a strong restriction for one-dimensional problems. For multi-dimensional problems in a tensor space, having this restriction can also be disturbing, especially under a sparse grid design. Many famous sparse grid designs are not based on equally spaced one-dimensional points, such as the Clenshaw-Curtis sparse grids \citep{gerstner1998numerical} or the ones suggested by \cite{plumlee2014fast}.

\noindent\textbf{Local Approximate Gaussian Processes}: \cite{gramacy2015local} proposed a sequential design scheme that dynamically defines the support of a Gaussian process predictor based on a local subset of the data. The local subset comprises of $m$ data points and, consequently, local approximate GP reduces the time and space complexity of GPs regression to $\CalO(m^3)$ and $\CalO(m^2+n)$ respectively. Local approximate GPs can achieve a decent accuracy level in empirical experiments but theoretical properties of this algorithm are still unknown.  

\noindent\textbf{Random Fourier Features}: The class of Fourier features methods originates from the work by \cite{rahimi2007random}. These methods essentially use $\sum_{i=1}^m \phi_i(\BFx)\phi_i(\BFx')$ to approximate $K(\BFx,\BFx')$, where $\phi_1(\BFx), \ldots,\psi_m(\BFx)$ are basis functions constructed based on random samples from the spectral density, i.e., the Fourier transform of the kernel function $K$. 
This low-rank approximation reduces the time and space complexity of GP regression to $\CalO_p(m^2n)$ and $\CalO_p(m^2+mn)$, respectively, with accuracy $\CalO_p(m^{-1/2})$ \citep{sriperumbudur2015optimal}. Clearly, the price for fast computation of random Fourier features is the loss of its accuracy. 

\noindent\textbf{Nystr\"om Approximation}: These methods approximate the $n\times n$ covariance matrix $\BFK$ by an $m\times m$ matrix $\widetilde{\BFK}=K(\widetilde{\BFX},\widetilde{\BFX})$, where $\widetilde{\BFX}=\{\tilde{\BFx}_i\}_{i=1}^m$ are called the inducing points. Similar to random Fourier features, Nystr\"om approximations reduce the time complexity and space complexity of GP regression to $\CalO_p(m^2n)$ and $\CalO_p(m^2+mn)$, respectively. There are several approaches to choose the inducing points. \cite{smola2000sparse,williams2001using}  selected $\widetilde{\BFX}$ from data points $\BFX$ by an orthogonalization procedure. \cite{titsias2009variational,bui2017unifying} treated $\widetilde{\BFX}$ as hidden variables and select these inducing points via variational Bayesian inference. \cite{Katzfuss_2017,ChenStein17} further developed the Nystr\"om approximation to construct more precise kernel approximations with multi-resolution structures. For Mat\'ern-$\nu$ kernel with $\nu>3/2$, it is shown in \cite{burt2019rates} that the accuracy level of any inducing points method is $\CalO_p(m^{-2\nu-1})$, which is higher than that of the random Fourier features. It was also shown in \cite{tuo2020kriging} that GP regression with  Mat\'ern-$\nu$ kernel converges to the underlying true GP at the rate $\CalO(n^{-\nu})$ and, hence, the number of inducing points should satisfy  $m=\CalO({n^{\frac{\nu}{2\nu+1}}})$ to achieve the optimal order of approximation accuracy. In this case, the time and space complexity of Nystr\"om approximations  are $\CalO(n^{1+\frac{2\nu}{2\nu+1}})$ and $\CalO(n^{1+\frac{\nu}{2\nu+1}})$, respectively. These are higher than that of the proposed algorithm, not to mention that the latter provides the exact solutions.

\noindent\textbf{Other methods}: Using a compactly supported kernel \citep{gramacy2020surrogates} can induce a sparse covariance matrix, which can lead to an improvement in matrix manipulations. However, 
if the support of the kernel remains the same while the design points become dense in a finite interval, the sparsity of the covariance matrix is not high enough to improve the order of magnitude. On the other hand, shrinking the support 
may substantially change the sample path properties of the GP, and impair the power of prediction. Recently, \cite{loper2021general} proposed a general approximation scheme for one-dimensional GP regression, which results in a linear-time inference method.

\section{Theory of Kernel Packet Basis}\label{sec:KP_theory}

In this section, we introduce the mathematical theory for the novel approach of inverting the correlation matrix in (\ref{eq:conditional-mean}) and (\ref{eq:conditional-variance}). Technical proofs of all theorems are deferred to Appendix \ref{appenA2}.

Direct inverting the matrix $\mathbf{K}$ in (\ref{eq:conditional-mean}) and (\ref{eq:conditional-variance}) is time consuming, because $\mathbf{K}$ is a dense matrix. %
Note that each entry of $\mathbf{K}$ is an evaluation of function $K(\cdot,x_j)$ for some $j$. The matrix $\mathbf{K}$ is not sparse because the support of $K$ is the entire real line.
The main idea of this work is to find an \textit{exact} representation of $\mathbf{K}$ in terms of sparse matrices.
This exact representation is built in terms of a change-of-basis transformation.

In this section, we suppose $K$ is a one-dimensional kernel.
Consider the linear space $\mathcal{K}=\operatorname{span}\{K(\cdot,x_j )\}_{j=1}^n$. The goal is to find another basis for $\mathcal{K}$, denoted as $\{\phi_j\}_{j=1}^n$, satisfying the following properties:
\begin{enumerate}
    \item Almost all of the $\phi_j$'s have \textit{compact supports}.
    \item $\{\phi_j\}_{j=1}^n$ can be obtained from $\{K(\cdot,x_j)\}_{j=1}^n$ via a \textit{sparse linear transformation}, i.e., the matrix defining the linear transform from $\{K(\cdot,x_j)\}_{j=1}^n$ to $\{\phi_j\}_{j=1}^n$ is sparse.
\end{enumerate}

Unless otherwise specified, throughout this article we assume that the one-dimensional kernel $K$ is a Mat\'ern correlation function as in (\ref{eq:Matern-1d}), whose spectral density is proportional to $\left(2\nu/\omega^2+x^2\right)^{-(\nu+1/2)}$; see \cite{Rasmussen06gaussianprocesses,tuo2016theoretical}. For notational simplicity, let $c^2:=2\nu/\omega^2$ and the above spectral density is proportional to $(c^2+x^2)^{-(\nu+1/2)}$.

\subsection{Definition and Existence of Kernel Packets}\label{sec:KP_def}

In this section we introduce the theory that explains how we can find a compactly supported function in $\mathcal{K}$. Clearly, such a function must admit the representation $\phi(x)=\sum_{j=1}^n A_j K(x,x_j).$ Recall the requirement that the linear transform is sparse, which means that most of the coefficients $A_j$'s must be zero. This inspires the following definition.

\begin{definition}\label{Def:KP}
Given a correlation function $K$ and input points $a_1<\cdots<a_k$, a non-zero function $\phi$ is called a \textit{kernel packet (KP) of degree $k$}, if it admits the representation
$\phi(x)=\sum_{j=1}^k A_j K(x,a_j),$    
and the support of $\phi$ is $[a_1,a_k]$.
\end{definition}

At first sight, it seems to be too optimistic to expect the existence of KPs. But, surprisingly, these functions do exist for one-dimensional Mat\'ern correlation functions with half-integer smoothness. We will show that if the smoothness parameter $\nu$ is a half integer, i.e., $\nu-1/2\in\mathbb{N}$, there is a KP of degree $k:=2\nu+2$ given any $k$ distinct input points.

For simplicity, we will use $k$ to parametrize the Mat\'ern correlation, in other words, $\nu=(k-2)/2$ for $k= 3,5,7,\ldots$. Let $\mathbf{a}=(a_1,...,a_k)^T$ be a vector with $a_1<\cdots <a_k$.  The goal is to find coefficients $A_j$'s such that
\begin{equation}\label{phi}
    \phi_{\mathbf{a}}(x):=\sum_{j=1}^kA_jK(x,a_j)
\end{equation}
is a KP.
We will first find a necessary condition for $A_j$'s, and next we will prove that such a condition is also sufficient.
We apply the Paley-Wiener theorem (see Lemma \ref{lem:Paley-Wiener theorem compact support} in Appendix \ref{appenA1} and \cite{stein2003complex}), which states that $\phi_\mathbf{a}(x)$ has a compact support only if the inverse Fourier transform of $\phi_\mathbf{a}$, denoted as $\tilde{\phi}_\mathbf{a}(x)$, can be extended to an entire function, i.e., a complex-valued function that is holomorphic on the whole complex plane. Let $i=\sqrt{-1}$. Our convention of inverse Fourier transform is $\tilde{f}(\xi)=(2\pi)^{-1/2}\int_{-\infty}^\infty f(x) e^{i\xi x}d x$. Direct calculations show
    $$\tilde{\phi}_\mathbf{a}(x)\propto \left [ \sum_{j=1}^{k} A_j \exp\{i a_j x\} \right](c^2+x^2)^{-(k-1)/2}, x\in\mathbb{R}.$$
Clearly, the analytic continuation of this function (up to a constant) is
    $$\tilde{\phi}_\mathbf{a}(z)\propto\left [ \sum_{j=1}^{k} A_j \exp\{i a_j z\} \right](c^2+z^2)^{-(k-1)/2}\eqqcolon\gamma(z)(c^2+z^2)^{-(k-1)/2},$$
and this function can be defined at any $z\in\mathbb{C}\setminus \{\pm ci\}$. Note that the function $(c^2+z^2)^{-(k-1)/2}$ has poles at $z=\pm ci$, each with multiplicity $(k-1)/2$. According to Paley-Wiener theorem, we have to make $\tilde{\phi}_\mathbf{a}(z)$ an entire function, which implies that $\gamma(\pm ci)=0$, each with multiplicity $(k-1)/2$ as well. This condition leads to a set of equations\footnote{This statement is formalized as Lemma \ref{lem:entire} in Appendix \ref{appenA2}.}:
\begin{eqnarray*}
	\gamma^{(j)}(ci)=0, && \gamma^{(j)}(-ci)=0,
\end{eqnarray*}
for $j=0,\ldots,(k-3)/2$, where $\gamma^{(j)}$ denotes the $j$th derivative of $\gamma$. Clearly, there are $k-1$ equations, which can be rewritten as
\begin{eqnarray}\label{eq:LE}
    \sum_{j=1}^{k} A_j a_j^l \exp\{\delta c a_j\}=0,
\end{eqnarray}
with $l=0,\ldots,(k-3)/2$ and $\delta=\pm 1$, which is a $(k-1)\times k$ linear system. All solutions to this system are real-valued vectors because all coefficients are real. 

Next we study the property of the linear system (\ref{eq:LE}) and the corresponding $\phi_\mathbf{a}$. Theorem \ref{theo:solutionspace} states that $\phi_\mathbf{a}$ can be uniquely determined by (\ref{eq:LE}) up to a multiplicative constant.

\begin{theorem}\label{theo:solutionspace}
	If $a_1, \ldots, a_k$ are distinct, the solution space of (\ref{eq:LE}) is one-dimensional, i.e., there do not exist two linearly independent solutions to (\ref{eq:LE}).
\end{theorem}

Another important property of (\ref{eq:LE}) is that its solution is not affected by a shift of $\mathbf{a}$. Define $\mathbf{a}+t=(a_1+t,\ldots,a_n+t)^T$.

\begin{theorem}\label{theo:invariant}
    The solution space of (\ref{eq:LE}), as a function of $\mathbf{a}$, is invariant under a shift transformation $T_t(\mathbf{a})=\mathbf{a}+t$ for any $t\in\mathbb{R}$.
\end{theorem}

\begin{remark}\label{rem:1}
Theorem \ref{theo:invariant} suggests that we can apply a shift on $\mathbf{a}$ without affecting the solution space. It is worth noting that, although the solution space does not change theoretically, the condition number of the linear system (\ref{eq:LE}) may change, which may significantly affect the numerical accuracy. In order to enhance the numerical stability in solving (\ref{eq:LE}), we suggest standardizing $\mathbf{a}$ using transformation $T_t(\mathbf{a})=\mathbf{a}+t$ such that $a_1+t=-(a_n+t)$, i.e., $t=-(a_1+a_n)/2$. The same standardization technique will be employed in the proof of Theorem \ref{theo:support1}.
\end{remark}

Theorem \ref{theo:support1} confirms that any non-zero $\phi_\mathbf{a}$ is indeed a KP.

\begin{theorem}\label{theo:support1}
    The support of any non-zero function $\phi_\mathbf{a}$ defined by (\ref{phi}) and (\ref{eq:LE}) is $[a_1,a_k]$.
\end{theorem}

In other words, we have the following Corollary \ref{coro:1}.

\begin{corollary}\label{coro:1}
Let $K$ be a Mat\'ern correlation with smoothness $\nu$. If $\nu$ is a half integer, then $K$ admits a KP with degree $2\nu+2$. In addition, given $a_1<\cdots<a_k$, function $\phi_\mathbf{a}$ with the form (\ref{phi}) is a KP if and only if the coefficients $A_j$'s are given by a non-zero solution to (\ref{eq:LE}).
\end{corollary}

Figure \ref{fig:KP} illustrates that the linear combination of $5$ components $\{K(\cdot,a_j)\}_{j=1}^5$ provides a compactly supported KP corresponding to Mat\'ern-$3/2$ correlation function.

\begin{figure}[h]
\centering
\includegraphics[width=0.6\textwidth]{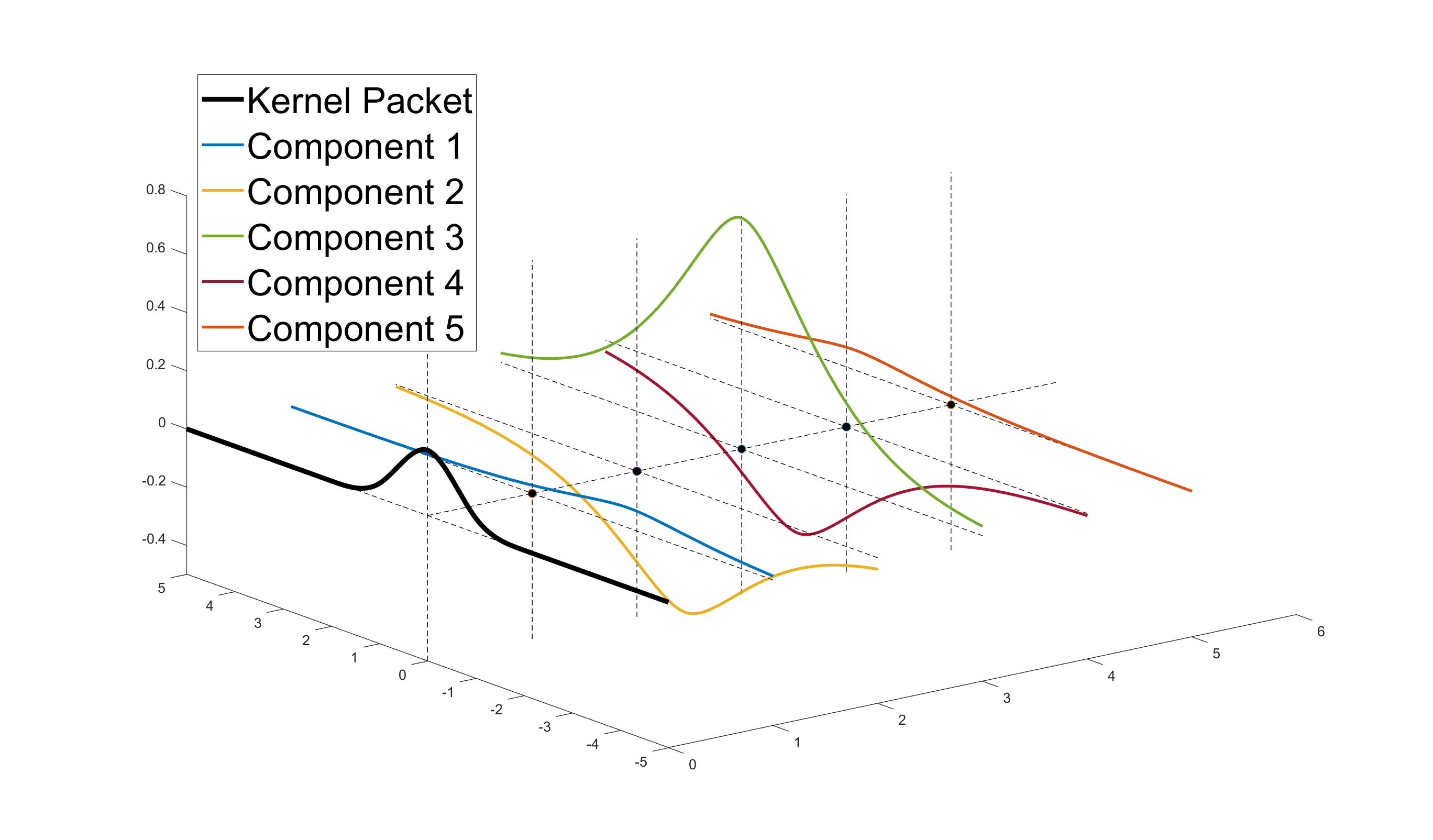}
\caption{KP $\phi_{\BFa}$ (black line) corresponding to Mat\'ern-$3/2$, and Mat\'ern-$3/2$ correlation function and its components $\{K(\cdot,a_j)\}_{j=1}^5$. \label{fig:KP}} 
\end{figure}

It is evident that KPs are highly \textit{non-trivial and precious}. Their existence relies on the correlation function. Theorem \ref{theo:matern_noKP} shows that many other correlation function do not admit any KP, and consequently, the proposed algorithm is not applicable to these correlations.
\begin{theorem}
    \label{theo:matern_noKP}
    The following correlation functions do not admit KPs:
    \begin{enumerate}
        \item Any Mat\'ern correlation function whose smoothness parameter is not a half integer.
        \item Any Gaussian correlation function.
    \end{enumerate}
\end{theorem}


Theorem \ref{theo:matern_KPdegree} shows that the KP constructed by (\ref{eq:LE}) has the lowest degree.

\begin{theorem}
    \label{theo:matern_KPdegree}
    Let $K$ be a Mat\'ern correlation function with half-integer smoothness $\nu$. Let $m$ be a positive integer with $m<2\nu+2$. Then any function of the form $\sum_{j=1}^m A_j K(\cdot,a_j)$ does not have a compact support unless $A_j=0$ for all $j=0,\ldots,m$, and in other words, there does not exist a KP of degree lower than $2\nu+2$.
\end{theorem}


   

\subsection{One-sided Kernel Packets}\label{sec:one_sided}
Besides KPs, we need to introduce a set of functions to capture the ``boundary effects'' of Gaussian process regression.
As before, let $\BFa=(a_1,...,a_s)^T$ be a vector with $a_1<\cdots <a_s$.
We consider the functions
\begin{equation}\label{phi_one_sided}
    \phi_{\BFa}(x):=\sum_{j=1}^sA_jK(x,a_j),
\end{equation}
with $(k+1)/2 \leq s \leq k-1$ and a non-zero real vector $(A_1,\ldots,A_s)^T$. Then Theorem \ref{theo:matern_KPdegree} suggests that $\phi_{\BFa}$ in (\ref{phi_one_sided}) cannot have a compact support. Nevertheless, it is possible that the support of $\phi_{\BFa}$ is a half real line. In this case, we cal $\phi_{\BFa}$ a one-sided KP. Specifically, we call $\phi_{\BFa}$ a \textit{right-sided KP} if $\supp\phi_{\BFa}=[a_1,+\infty)$, and we call $\phi_{\BFa}$ a \textit{left-sided KP} if $\supp\phi_{\BFa}=(-\infty,a_s]$.

First we consider right-sided KPs.
We propose to identify $A_j$'s by solving
\begin{equation}
    \sum_{j=1}^{s} A_ja_j^{l} \exp\{-ca_j\}=0,\quad \sum_{j=1}^{s} A_j a_j^{r} \exp\{ca_j\}=0,
\label{eq:LE_one_sided_R}
\end{equation}
where $l=0,\ldots,(k-3)/2$ and the second term of \eqref{eq:LE_one_sided_R} comprises auxiliary equations for the case $s\geq (k+3)/2$ with $r=0,\ldots,s-(k+3)/2$. Similar to \eqref{eq:LE}, \eqref{eq:LE_one_sided_R} is an $(s-1)\times s$ linear system.

The following theorems describes the properties of the linear system \eqref{eq:LE_one_sided_R} and the corresponding $\phi_{\BFa}$. Specifically, Theorem \ref{theo:support1_one_sided} confirms that $\phi_{\BFa}$ is indeed a right-sided KP. 
\begin{theorem}\label{theo:solutionspace_one_sided}
    The solution space of \eqref{eq:LE_one_sided_R} is one-dimensional provided that $a_1,\ldots,a_s$ are distinct.
\end{theorem}

\begin{theorem}\label{theo:shift_one_sided}
        The solution space of \eqref{eq:LE_one_sided_R}, as a function of $\mathbf{a}$, is invariant under a shift transformation $T_t(\mathbf{a})=\mathbf{a}+t$ for any $t\in\mathbb{R}$.
\end{theorem}

\begin{theorem}\label{theo:support1_one_sided}
    The support of any non-zero function $\phi_{\BFa}$ defined by \eqref{phi_one_sided} and \eqref{eq:LE_one_sided_R} is $[a_1,+\infty)$.
\end{theorem}

Left-sided KPs are constructed similarly by solving the following equations:
\begin{equation}
\label{eq:LE_one_sided}
    \sum_{j=1}^{s} A_j a_j^{l} \exp\{ca_j\}=0,\quad \sum_{j=1}^{s} A_j a_j^{r} \exp\{-ca_j\}=0,
\end{equation}
where $l=0,\ldots,(k-3)/2$ and the second term comprises auxiliary equations for the case $s\geq (k+3)/2$ with $r=0,\ldots,s-(k+3)/2$. The properties of left-sided KPs are analogous to these stated in Theorems \ref{theo:solutionspace_one_sided}-\ref{theo:support1_one_sided}, for which we omit the statements.

\begin{remark}
As in Remark \ref{rem:1}, we suggest applying a shift transformation on $\mathbf{a}$ before computing $A_j$'s. Let $T_t(\mathbf{a})=(a'_1,\ldots,a'_s)^T.$ We suggest using $T_t$ such that $a'_1=0$ (i.e., $t=-a_1$) for the right-sided KPs, and $a'_s=0$ (i.e., $t=-a_s$) for the left sided KPs. The same shifting is employed in the proof of Theorem \ref{theo:support1_one_sided}.
\end{remark}

\subsection{Kernel Packet Basis} \label{sec:basis}
Let $x_1<\cdots<x_n$ be the input data, and $K$ a Mat\'ern correlation function with a half-integer smoothness. Suppose $n\geq k$.
We can construct the following $n$ functions, as a subset of $\CalK$:
\begin{enumerate}
    \item $\phi_1,\phi_2,\ldots,\phi_{(k-1)/2}$, defined as left-sided KPs $\phi_{(x_1,\ldots,x_{(k+1)/2})},\phi_{(x_1,\ldots,x_{(k+1)/2+1})},\ldots, \phi_{(x_1,\ldots,x_{k-1})}$,
    \item  $\phi_{(k+1)/2},\phi_{(k+1)/2+1},\ldots,\phi_{n-(k-1)/2}$, defined as KPs $\phi_{(x_1,\ldots,x_{k})},\phi_{(x_2,\ldots,x_{k+1})},\ldots,\phi_{(x_{n-k+1},\ldots,x_n)}$,
    \item  $\phi_{n-(k-3)/2},\ldots,\phi_{n-1},\phi_n$, defined as right-sided KPs  $\phi_{(x_{n-k+2},\ldots,x_n)},\ldots,\phi_{(x_{n-(k-1)/2-1},\ldots,x_n)},\phi_{(x_{n-(k-1)/2},\ldots,x_n)}$.
\end{enumerate}
Note that KPs and one-sided KPs given the input points cannot be uniquely defined. They are unique only up to a non-zero multiplicative factor. Here the choice of these factors are nonessential. The general theory and algorithms in this article will be valid for each specific choice.
Now we present Theorem \ref{theo:linear_indpendent}, which, together with the fact that the dimension of $\CalK$ is $n$, implies that $\{\phi_{j}\}_{j=1}^n$ forms a basis for $\CalK$, referred to as the \textit{KP basis}.
\begin{theorem}    \label{theo:linear_indpendent}
    Let $x_1<\cdots<x_n$ be the input data and the functions $\phi_1,\ldots,\phi_n$ are constructed in the above manner.
    Then the basis functions $\{\phi_{j}\}_{j=1}^n$ are linearly independent in $\CalK$.
\end{theorem}
Further, it is straightforward to check via Theorems \ref{theo:support1} and \ref{theo:support1_one_sided} that, given any $x\in\Real$, the vector $\boldsymbol{\phi}(x)=\left(\phi_1(x),\ldots,\phi_n(x)\right)^T$ has at most $k-1$ non-zero entries. As a result, we have constructed a basis for $\CalK$ satisfying the two sparse properties mentioned at the beginning of Section \ref{sec:KP_theory}. Figure \ref{fig:KP_basis} illustrates a KP basis corresponding to Mat\'ern-$3/2$ and Mat\'ern-$5/2$ correlation function with input points $\BFX=\{0.1,0.2,\ldots,1\}$.

\begin{figure}[ht]
\centering
\includegraphics[width=0.4\textwidth]{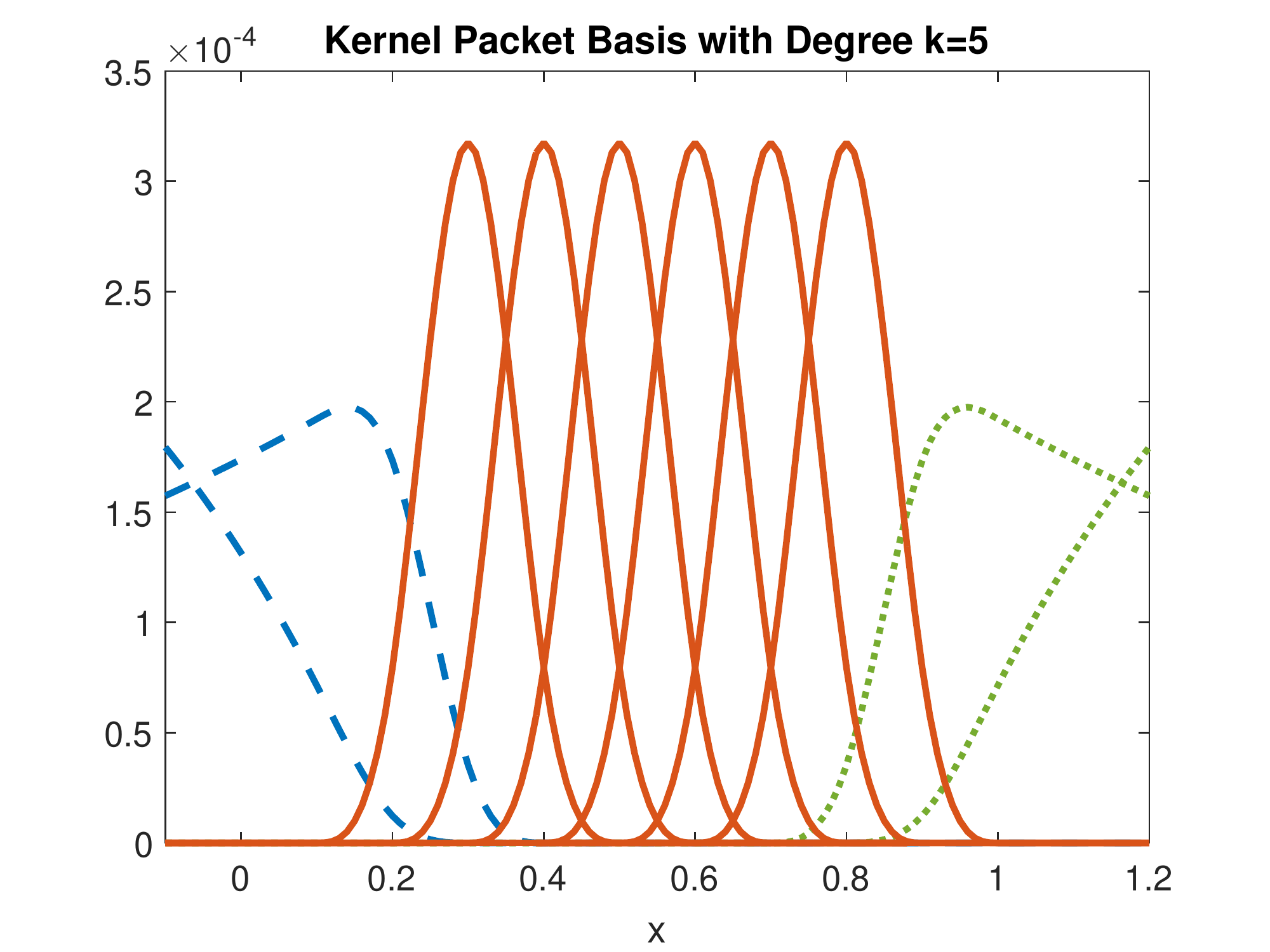}
\includegraphics[width=0.4\textwidth]{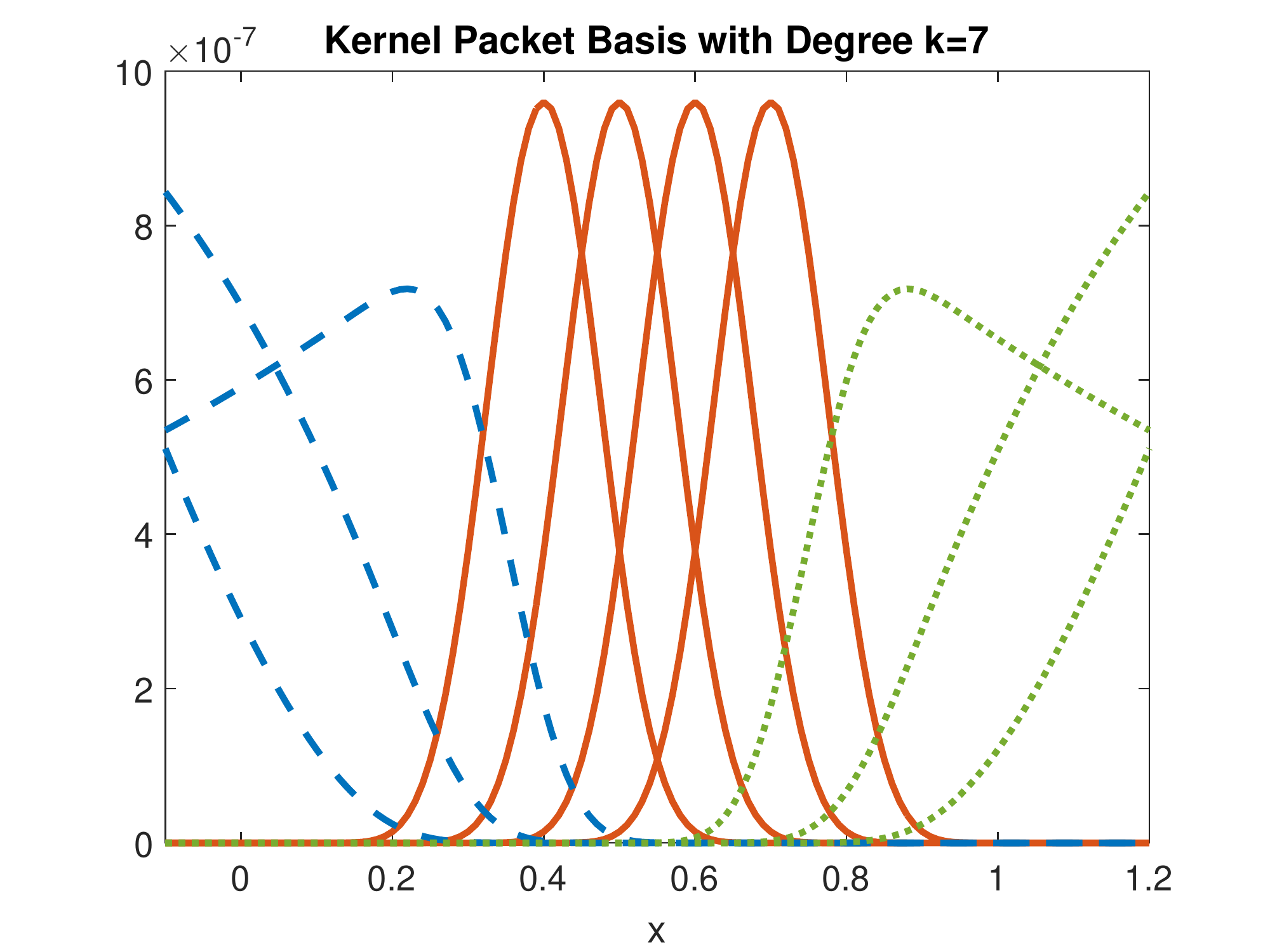}
\caption{KP basis functions corresponding to  Mat\'ern-$3/2$ (left) and Mat\'ern-$5/2$ (right) correlation function \label{fig:KP_basis}with input points $\BFX=\{0.1,0.2,\ldots,1\}.$ The KPs, left-sided KPs, and the right-sided KPs are plotted in orange, blue, and green lines, respectively.} 
\end{figure}

\section{Kernel Packet Algorithms}\label{sec:algorithm}

In this section, we will employ the KP bases to develop scalable algorithms for GP regression problems. In Sections \ref{sec:algorithm-1D-noiseless} and \ref{sec:algorithm-1D-noisy} , we present algorithms for one-dimensional GP regression with noiseless and noisy data, respectively. In section \ref{sec:algorithm-hidg-dim}, we generalize the one-dimensional algorithms to higher dimensions by applying the tensor and sparse grid techniques.

\subsection{One-dimensional GP Regression with Noiseless Data}\label{sec:algorithm-1D-noiseless}




The theory in Section \ref{sec:KP_theory} shows that for one-dimensional problems, given ordered and distinct inputs $x_1<\cdots<x_n$, the correlation matrix $\BFK$ admits a sparse representation as
\begin{equation}
    \label{eq:Kmat2KPmat}
    \BFK \BFA=\BFphi(\BFX),
\end{equation}
where both $\BFA$ and $\BFphi(\BFX)$ are \emph{banded matrices}.
In \eqref{eq:Kmat2KPmat}, the $(l,j)^{\rm th}$ entry of $\BFphi(\BFX)$ is $\phi_j(x_l)$. 
In view of the compact supportedness of $\phi_j$,
$\BFphi(\BFX)$ is a banded matrix with bandwidth $(k-3)/2$:
\begin{equation*}
\BFphi(\BFX)=
    \begin{bmatrix}
        \ddots&&&&\\
        \ddots&\phi_{j-\frac{k-3}{2}}(x_{j-2\frac{k-3}{2}} )&&&\\
        \ddots&\vdots&\ddots&&\\
        &\phi_{j-\frac{k-3}{2}}(x_{j} )&\cdots&\phi_{j+\frac{k-3}{2}} (x_{j} )&\\
        &&\ddots&\vdots&\ddots\\
        &&&\phi_{j+\frac{k-3}{2}}(x_{j+2\frac{k-3}{2}} )&\ddots\\
        &&&&\ddots
    \end{bmatrix}.
\end{equation*}

The matrix of $\BFA$ consists of the coefficients to construct the KPs.
In view of the sparse representation,
$\BFA$ is a banded matrix with bandwidth $(k-1)/2$:
\begin{equation*}
\BFA=
    \begin{bmatrix}
        \ddots&&&&\\
        \ddots&A_{j-2\frac{k-1}{2},j-\frac{k-1}{2}}&&&\\
        \ddots&\vdots&\ddots&&\\
        &A_{j,j-\frac{k-1}{2}}&\cdots&A_{j,j+\frac{k-1}{2}}&\\
        &&\ddots&\vdots&\ddots\\
        &&&A_{j+2\frac{k-1}{2},j+\frac{k-1}{2}}&\ddots\\
        &&&&\ddots
    \end{bmatrix}.
\end{equation*}

Computing $\BFA$ and $\BFphi(\BFX)$ takes $\CalO(k^3n)$ time, because in the construction of each $\phi_j$, at most $k$ kernel basis functions are needed and the time complexity for solving the coefficients $\{A_{w,j}: |w-j|\leq \frac{k-1}{2}\}$ satisfying equation \eqref{eq:LE}, \eqref{eq:LE_one_sided_R} or \eqref{eq:LE_one_sided} is $\CalO(k^3)$. The computational time $\CalO(k^3n)$ in this step will dominate that in the next step, which is $\CalO(k^2 n)$. However, if the design points are equally spaces, the KP coefficients given by (\ref{eq:LE}) will remain the same for each $k$ consecutive data points, so that we only need to compute these values once, and thus the computational time of this step is only $\CalO(k^4)$. In this case, the computation time in the next step, i.e., $\CalO(k^2 n)$, will be dominant, provided that $k\ll n$.

Now we solve the GP regression problem,
by substituting the identity $K(\cdot,\BFX)=\BFphi(\cdot)\BFA^{-1}$ into \eqref{eq:conditional-mean} and \eqref{eq:conditional-variance} to obtain
\begin{align}
          &\E\left[Y(x^*)\big|\BFY\right]= \mu(x^*)+\BFphi^T(x^*)\left[\BFphi(\BFX)\right]^{-1}\left(\BFY-\boldsymbol{\mu}\right),\label{eq:conditional-mean-KP}\\
          & \Var\left[Y(x^*)\big|\BFY\right]=\sigma^2\bigg(K(x^*,x^*)-\BFphi^T(x^*)\left[\BFphi(\BFX)\right]^{-1}K(\BFX,x^*)\bigg).\label{eq:conditional-variance-KP}
\end{align}
The key to GP regression now becomes calculating the vector $\left[\BFphi(\BFX)\right]^{-1}{\mathbf{v}}$ with $\mathbf{v}=\BFY-\boldsymbol{\mu}$ or $\mathbf{v}=K(\BFX,x^*)$. This is equivalent to solving the sparse banded linear system $\BFphi(\BFX)\mathbf{s}=\mathbf{v}$.  There exists quite a few sparse linear solvers that can solve this linear system efficiently. For example, the algorithm based on the LU decomposition in \cite{davis2006direct} can be applied to solve  for $\mathbf{s}$ in $\CalO(k^2n)$ time.  MATLAB provides convenient and efficient builtin functions, such as \texttt{mldivide} or \texttt{decomposition}, to solve sparse banded linear system in this form.

It is worth noting that (\ref{eq:conditional-mean-KP}) can be executed in the following faster way when we need to evaluate $\E\left[Y(x^*)|\BFY\right]$ for a many different $x^*$. First, we compute $\mathbf{s}:=\left[\BFphi(\BFX)\right]^{-1}\left(\BFY-\boldsymbol{\mu}\right)$, which takes $\CalO(k^2 n)$ time. Next we evaluate $\mu(x^*)+\BFphi^T(x^*)\mathbf{s}$ for different $x^*$. As said before, $\BFphi^T(x^*)$ has at most $k-1$ non-zero entries; see Figure \ref{fig:KP_basis}. If we know which $k-1$ entries are non-zero, the second step takes only $\CalO(k)$ time. To find the non-zero entry, a general approach is to use a binary search, which takes $\CalO(\log n)$ time. Sometime, these entries can be found within a constant time. For example, if the design points are equally spaced, there exist explicit expressions for the indices of the non-zero entries; if we need to predict for $x^*$ over a dense mesh (which is a typical task of surrogate modeling), we can use the indices of the non-zero entries for the previous point as an initial guess to find those for the current point. 

Similar to the conditional inference, the log-likelihood function \eqref{eq:loglike} can also be computed in $\CalO(k^2n)$ time. First, the log-determinant of $\BFK$ can be rewritten as $\log \det (\BFK)=\log \det\left(\BFphi(\BFX)\right)-\log \det(\BFA)$, according to identity \eqref{eq:Kmat2KPmat}:
Because both $\BFA$ and $\BFphi(\BFX)$ are banded matrices, their determinants can be computed in $\CalO(k^2n)$ time by sequential methods \citep[section 4.1]{kamgnia2014some}. Second, the same method for the conditional inference can be applied to compute  $\left(\BFY-\mathbf{F}\boldsymbol{\beta}\right)^T\BFK^{-1}\left(\BFY-\mathbf{F}\boldsymbol{\beta}\right)$ in $\CalO(k^2n)$ time. 

\subsection{One-dimensional GP Regression  with Noisy Data}\label{sec:algorithm-1D-noisy}
Suppose we observe data $\BFZ$, which is a noisy version of $\BFY$. Specifically, $Z(\BFx_i)=Y(\BFx_i)+\varepsilon$, where $\varepsilon\sim \mathcal{N}(0,\sigma^2_Y)$. In this case, the covariance of the
observed noisy responses is $\text{Cov}\big(Z(\BFx_i),Z(\BFx_j)\big)=\sigma^2K(\BFx_i,\BFx_j)+\sigma_Y^2\mathbb{I}(\BFx_i=\BFx_j)$. In other words, the covariance matrix $\text{Cov}(\BFZ,\BFZ)$ is $\sigma^2\BFK+\sigma_Y^2\bold{I}$, where $\bold{I}$ is the identity matrix. The posterior predictor at a new point $\BFx^*$ is also normal distributed with the following conditional mean and variance:
\begin{align}
          &\E\left[Y(\BFx^*)\big|\BFZ\right]= \mu(\BFx^*)+K(\BFx^*,\BFX)\left[\BFK+\frac{\sigma_Y^2}{\sigma^2}\bold{I}\right]^{-1}\left(\BFZ-\boldsymbol{\mu}\right),\label{eq:conditional-mean-noisy}\\
          & \Var\left[Y(\BFx^*)\big|\BFZ\right]=\sigma^2\left(K(\BFx^*,\BFx^*)-K(\BFx^*,\BFX)\left[\BFK+\frac{\sigma_Y^2}{\sigma^2}\bold{I}\right]^{-1}K(\BFX,\BFx^*)\right),\label{eq:conditional-variance-noisy}
\end{align}
and the log-likelihood function given data $\BFZ$ is:
\begin{equation}\label{eq:loglike-noisy}
    L(\BFbeta,\sigma^2,\BFomega)=-\frac{1}{2}\left[ \log \det(\sigma^2\BFK+\sigma^2_Y\bold{I}) + \big(\BFZ-\mathbf{F}\BFbeta\big)^T\big[ \sigma^2\BFK+\sigma_Y^2\bold{I}\big]^{-1}\big(\BFZ-\mathbf{F}\BFbeta\big)\right].
\end{equation}

When the input $\BFx$ is one dimensional, \eqref{eq:conditional-mean-noisy}, \eqref{eq:conditional-variance-noisy}, and \eqref{eq:loglike-noisy} can be calculated in $\CalO(k^2 n)$ as the noiseless case because the covariance matrix $\sigma^2\BFK+\sigma_Y^2\bold{I}$ admits the following factorization:
\begin{equation}
     \label{eq:Kmat2KPmat-noisy}
     \sigma^2\BFK+\sigma_Y^2\bold{I}=\big(\sigma^2\BFphi(\BFX)+\sigma_Y^2\BFA\big)\BFA^{-1}.
\end{equation}
By substituting \eqref{eq:Kmat2KPmat-noisy} and the identity $K(\cdot,\BFX)=\BFphi(\cdot)\BFA^{-1}$ into \eqref{eq:conditional-mean-noisy}, \eqref{eq:conditional-variance-noisy}, and \eqref{eq:loglike-noisy}, we can obtain:
\begin{align}
          &\E\left[Y(x^*)\big|\BFZ\right]= \mu(x^*)+\BFphi^T(x^*)\left[\BFphi(\BFX)+\frac{\sigma_Y^2}{\sigma^2}\BFA\right]^{-1}\left(\BFZ-\boldsymbol{\mu}\right),\label{eq:conditional-mean-KP-noisy}\\
          & \Var\left[Y(x^*)\big|\BFZ\right]=\sigma^2\bigg(K(x^*,x^*)-\BFphi^T(x^*)\left[\BFphi(\BFX)+\frac{\sigma_Y^2}{\sigma^2}\BFA\right]^{-1}K(\BFX,x^*)\bigg).\label{eq:conditional-variance-KP-noisy}
\end{align}
and 
\begin{align}
    L(\BFbeta,\sigma^2,\BFomega)=&-\frac{1}{2}\bigg[ \log \det\big(\sigma^2\BFphi(\BFX)+\sigma_Y^2\BFA\big)-\log \det(\BFA)\nonumber\\
    &+ \big(\BFZ-\mathbf{F}\BFbeta\big)^T\BFA\big[\sigma^2 \BFphi(\BFX)+\sigma_Y^2\BFA\big]^{-1}\big(\BFZ-\mathbf{F}\BFbeta\big)\bigg].\label{eq:loglike-noisy-KP} 
\end{align}
We have shown that $\BFphi(\BFX)$ and $\BFA$ are banded matrices with bandwidth $(k-3)/2$ and $(k-1)/2$, respectively. Therefore, the matrix $\sigma^2 \BFphi(\BFX)+\sigma_Y^2\BFA$ is also a banded matrix with bandwidth $(k-3)/2$. Time complexity for computing this sum is $\CalO(k n)$. We then can use the algorithms for banded matrices introduced in section \ref{sec:algorithm-1D-noiseless} to compute  \eqref{eq:conditional-mean-KP-noisy}, \eqref{eq:conditional-variance-KP-noisy}, and \eqref{eq:loglike-noisy-KP} in time complexity $\CalO(k^2n)$. Recall that the time complexities for computing $\BFphi(\BFX)$ and $\BFA$ are both $\CalO(k^3n)$. Therefore, in the noisy setting, the total time complexity for computing the posterior and MLE is still $\CalO(k^3n)$, which is the same as the noiseless case. 

\subsection{Multi-dimensional KP}\label{sec:algorithm-hidg-dim}
When data is noiseless, the exact algorithm proposed in Section \ref{sec:algorithm-1D-noiseless} can be used to solve multi-dimensional problems if the input points are full or sparse grids.

A full grid is defined as the \emph{Cartesian product} of one dimensional point sets{: $\BFX^{\rm FG}=\times_{j=1}^d\BFX^{(j)}$ where each $\BFX^{(j)}$ denotes any one-dimensional point set.}
Assuming a  separable correlation function (\ref{eq:separable-kernel}) comprising $d$ one-dimensional Mat\'ern correlation functions with half-integer smoothness, and inputs on a full grid $\BFX^{\rm FG}$, the covariance vector $K(\BFx^*,\BFX^{\rm FG})$ and covariance matrix $\BFK$ decompose into
\emph{Kronecker products} of matrices over each input dimension \citep{saatcci2012scalable,wilson2014covariance}:
\begin{align}
    &K(\BFx^*,\BFX^{\rm FG})=\bigotimes_{k=1}^d K_j(x^*_j,\BFX^{(j)})=\bigotimes_{j=1}^d\BFphi^T_{j}(x_j^*)\BFA^{-1}_j=\bigg(\bigotimes_{j=1}^d\BFphi^T_{j}(x_j^*)\bigg)\bigg(\bigotimes_{j=1}^d\BFA^{-1}_j\bigg)\label{eq:covarince-vec-tensor}\\
    &\BFK=\bigotimes_{k=1}^d K_j(\BFX^{(j)},\BFX^{(j)})=\bigotimes_{j=1}^d\BFphi_{j}(\BFX^{(j)})\BFA^{-1}_j=\bigg(\bigotimes_{j=1}^d\BFphi_{j}(\BFX^{(j)})\bigg)\bigg(\bigotimes_{j=1}^d\BFA^{-1}_j\bigg).\label{eq:covarince-mat-tensor}
\end{align}
When we compute the vector $K(\BFx^*,\BFX)\BFK^{-1}$, the  matrix  $\bigotimes_{j=1}^d\BFA^{-1}_j$ is  cancelled  as the one dimensional case. Therefore, \eqref{eq:conditional-mean} and \eqref{eq:conditional-variance}  can be expressed as 

\begin{align}
        &\E\left[Y(\BFx^*)\big|\BFY\right]= \mu(\BFx^*)+\left(\bigotimes_{j=1}^d\BFphi^T_{j}(x_j^*)\right)\left(\bigotimes_{j=1}^d\left[\BFphi_{j}(\BFX^{(j)})\right]^{-1}\right)\left(\BFY-\boldsymbol{\mu}\right)\label{eq:conditional-mean-tensor}\\
        & \Var\left[Y(\BFx^*)\big|\BFY\right]=\sigma^2\bigg(K(\BFx^*,\BFx^*)-\prod_{j=1}^d\BFphi^T_{j}(x_j^*)\left[\BFphi_{j}(\BFX^{(j)})\right]^{-1}K_j(\BFX^{(j)},x_j^*)\bigg)\label{eq:conditional-variance-tensor}
\end{align}
and the log-likelihood function \eqref{eq:loglike} becomes 
\begin{align}
    L(\BFbeta,\sigma^2,\BFomega)&=
    -\frac{1}{2}\bigg[n\log \sigma^2 + \sum_{j=1}^d\frac{n}{n_{j}}\left(\log \det\BFphi_{j}(\BFX^{(j)})-\log\det \BFA_{j}\right)\nonumber\\
    &+\frac{1}{\sigma^2} \big(\BFY-\mathbf{F}\BFbeta\big)^T
    \left(\bigotimes_{j=1}^d\BFA_{j}\right)\left(\bigotimes_{j=1}^d\left[\BFphi_{j}(\BFX^{(j)})\right]^{-1}\right)\big(\BFY-\mathbf{F}\BFbeta\big)\bigg],\label{eq:loglike-tensor}
\end{align}
where $\BFphi_{j}$'s are the KPs associated to correlation function $K_j$ and point set $\BFX^{(j)}$, $\BFA_{j}$ is the coefficient matrix for constructing $\BFphi_{j}$ defined in \eqref{eq:Kmat2KPmat}, and $n=\prod_{j=1}^dn_j$ is the size of $\BFX^{\rm FG}$, $n_{j}$ is the size of $\BFX^{(j)}$. We can also note that entries of vector $\bigotimes_{j=1}^d\BFphi_j(\cdot)$ are products of one-dimensional KPs. Therefore, similar to the one-dimensional case, $\bigotimes_{j=1}^d\BFphi_j(\cdot)$ is a vector of compactly supported functions.

\begin{figure}[ht]
\centering
\includegraphics[width=0.45\textwidth]{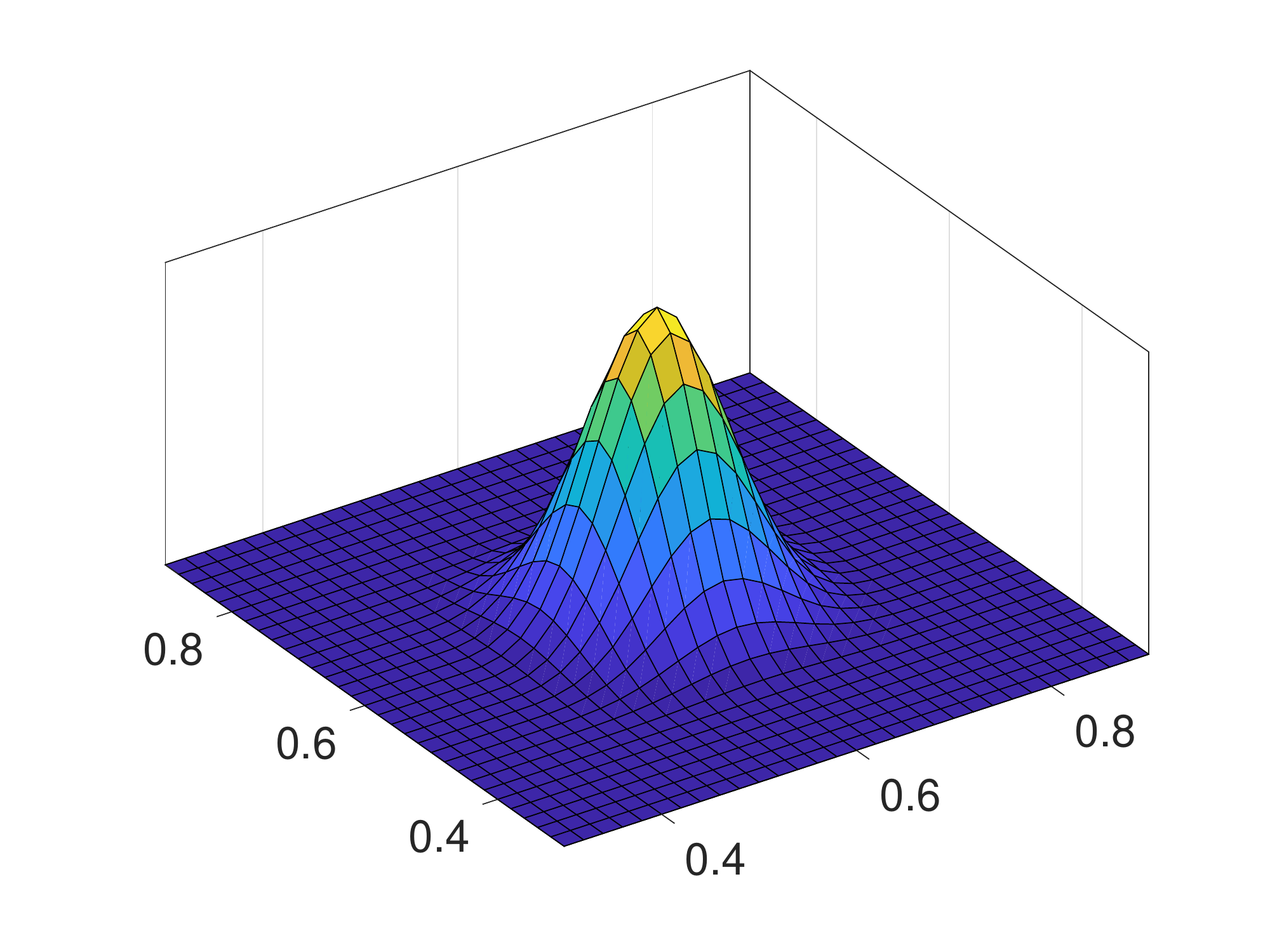}
\includegraphics[width=0.45\textwidth]{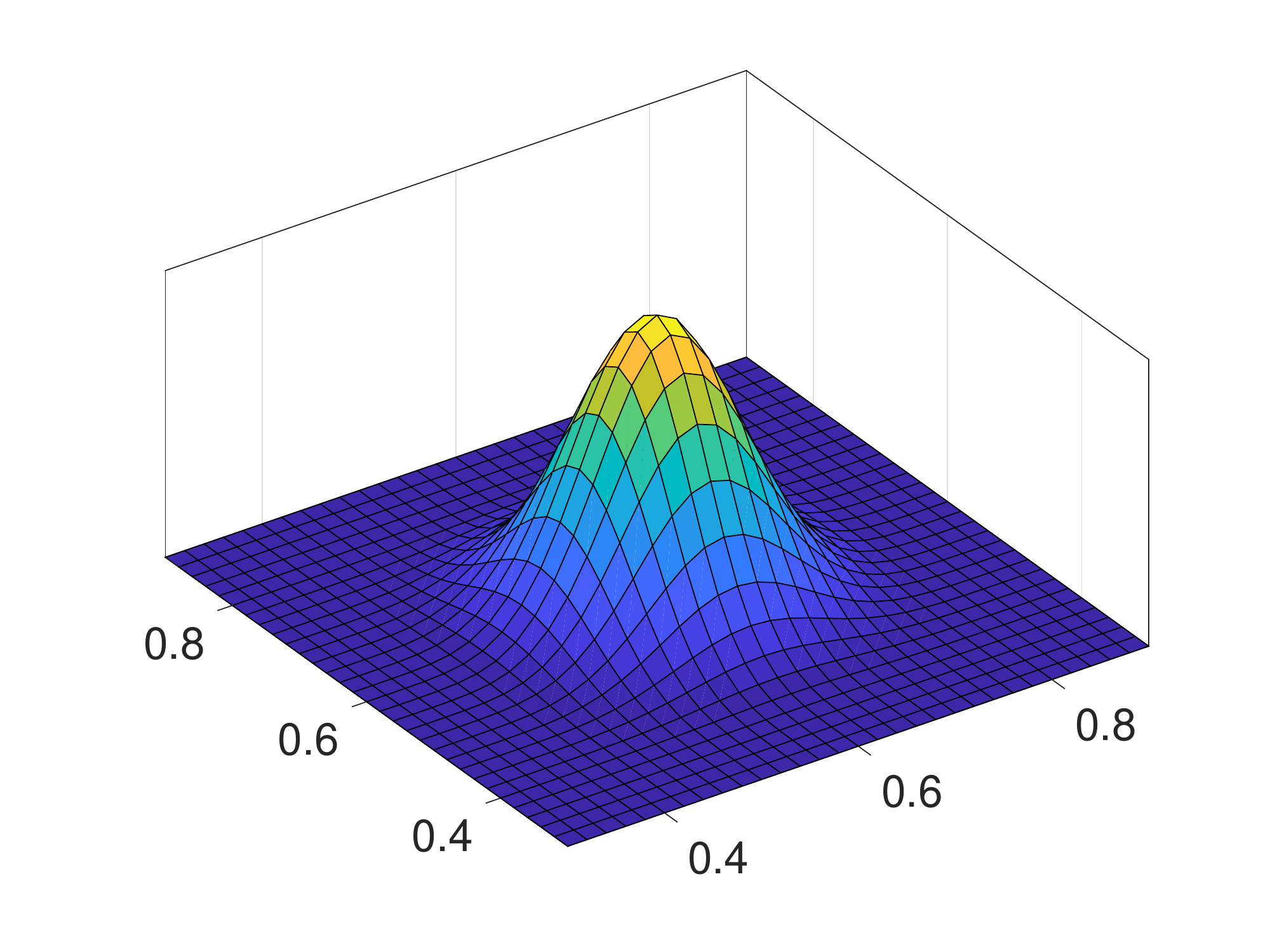}
\caption{product KP basis functions $\phi_{(.4\ .5\ .6\ .7\ .8)}(x_1)\phi_{(.4\ .5\ .6\ .7\ .8)}(x_2)$ corresponding to  Mat\'ern-$3/2$ (left) and $\phi_{(.3 \ .4\ .5\ .6\ .7\ .8\ .9)}(x_1)\phi_{((.3 \ .4\ .5\ .6\ .7\ .8\ .9)}(x_2)$ corresponding to  Mat\'ern-$5/2$ (right) correlation function. \label{fig:KP_basis_tensor}} 
\end{figure}

In \eqref{eq:conditional-mean-tensor}--\eqref{eq:loglike-tensor},  $\prod_{j=1}^d\BFphi^T_{j}(x_j^*)\left[\BFphi_{j}(\BFX^{(j)})\right]^{-1}K_j(\BFX^{(j)},x_j^*)$ and  $\sum_{j=1}^d\frac{n}{n_{j}}(\log \det\BFphi_{j}(\BFX^{(j)})-\log\det \BFA_{j})$ can clearly be computed in $\CalO(\sum_{j=1}^d k^3n_{j})$ time. The computation of $\left(\bigotimes_{j=1}^d\left[\BFphi_{j}(\BFX^{(j)})\right]^{-1}\right)\bold{v}$ for $\bold{v}\in\Real^n$ is known as \emph{Kronecker product least squares} (KLS),  which has been extensively studied; see, e.g., \cite{graham2018kronecker,fausett1994large}. Essentially, the task can be accomplished by sequentially solving $d$  problems where the $j^{\rm th}$ problem is to solve $n/n_j$ independent banded linear equations with $n_j$ variables.  With a banded solver, the total time complexity is  
$\CalO(dk^3n)$. 


Although full grid designs result in simple and fast computation of GP regression, their sizes increase exponentially in the dimension. When the dimension is large, another class of grid-based designs called the \emph{sparse grids} can be practically more useful. { Let $\BFX_{\bold{l}}$ denote  full grids in the form $\BFX_{\bold{l}}=\times_{j=1}^d\BFX_{l_j}$ where $l_j\in\mathbb{N}$ and $\BFX_{l_j}$ is a one-dimensional point set satisfying the nested structure $\emptyset=\BFX_0\subseteq\BFX_1\subseteq\cdots\subseteq\BFX_{l_j}$ for each $j$. 
A sparse grid of level $\eta$ is defined as a union of full grids $\BFX_{\BFl}$'s:
$\BFX^{\rm SG}_{\eta}=\bigcup_{|\BFl|\leq \eta+d-1}\BFX_{\BFl}$,
where $|\BFl|:=\sum_{j=1}^dl_j$.} GP regression on sparse grid designs was first discussed in \cite{plumlee2014fast}. According to Algorithm 1 in \cite{plumlee2014fast}, \eqref{eq:conditional-mean-tensor} and \eqref{eq:conditional-variance-tensor}, GP regression on $\BFX^{\rm SG}_\eta$ admits the expression
\begin{align}
        &\E\left[Y(\BFx^*)\big|\BFY\right]= \mu(\BFx^*)+\sum_{|\BFl|=\max\{d,\eta-d+1\}}^{\eta}(-1)^{\eta-|\BFl|}{d-1 \choose |\eta|-\BFl}\bar{f}_\BFl(\BFx^*)\label{eq:conditional-mean-SG},\\
        & \Var\left[Y(\BFx^*)\big|\BFY\right]=\sigma^2\bigg(K(\BFx^*,\BFx^*)-\sum_{|\BFl|=\max\{d,\eta-d+1\}}^{\eta}(-1)^{\eta-|\BFl|}{d-1 \choose |\eta|-\BFl}\bar{K}_{\BFl}(\BFx^*,\BFx^*)\bigg)\label{eq:conditional-variance-SG},
\end{align}
where 
\begin{eqnarray*}
\bar{f}_\BFl:=\left(\bigotimes_{j=1}^d\BFphi^T_{l_j}(x_j^*)\right)\left(\bigotimes_{j=1}^d\left[\BFphi_{l_j}(\BFX_{l_j})\right]^{-1}\right)\left(\BFY_\BFl-\boldsymbol{\mu}_\BFl\right),\\
\bar{K}_\BFl(\BFx^*,\BFx^*):=\prod_{j=1}^d\BFphi^T_{l_j}(x_j^*)\left[\BFphi_{l_j}(\BFX_{l_j})\right]^{-1}K_j(\BFX_{l_j},x_j^*),
\end{eqnarray*}
come from \eqref{eq:conditional-mean-tensor} and \eqref{eq:conditional-variance-tensor}, respectively; $\BFY_\BFl$ and $\boldsymbol{\mu}_\BFl$ denote the sub-vectors of $\BFY$ and $\boldsymbol{\mu}$ on full grid $\BFX_\BFl$, respectively. Based on Theorem 1 and Algorithm 2 in \cite{plumlee2014fast} and \eqref{eq:loglike-tensor},  $\log \det \BFK$ and $(\BFY-\mathbf{F}\BFbeta\big)^{T}\BFK^{-1} (\BFY-\mathbf{F}\BFbeta\big)^{T}$ in the log-likehood function \eqref{eq:loglike} can be decomposed as the following linear combinations, respectively:
\begin{align}
    &\sum_{|\BFl|=\max\{d,\eta+d-1\}}^\eta\sum_{j=1}^d\bigg(\log \frac{\det\BFphi_{l_j}(\BFX_{l_j})}{ \det\BFphi_{l_j}(\BFX_{l_j-1})}-\log \frac{\det \BFA_{l_j}}{\det \BFA_{l_j-1}}
    \bigg)\prod_{w\neq j}(n_{l_w}-n_{l_w-1}), \label{eq:log-det-SG}\\
   &\sum_{|\BFl|=\max\{d,\eta+d-1\}}^\eta\frac{(-1)^{\eta-|\BFl|}}{\sigma^2}{d-1 \choose |\eta|-\BFl} \big(\BFY_\BFl-\mathbf{F}_{\BFl}\BFbeta\big)^T
    \left(\bigotimes_{j=1}^d\BFA_{l_j}\right)\left(\bigotimes_{j=1}^d\left[\BFphi_{l_j}(\BFX_{l_j})\right]^{-1}\right)\big(\BFY_\BFl-\mathbf{F}_\BFl\BFbeta\big),\label{eq:quadratic-SG}
\end{align}
where $\BFphi_{0}(\BFX_{0})=\BFA_0=1$  and $\mathbf{F}_\BFl$ denotes the sub-matrix of $\mathbf{F}$ on full grid $\BFX_\BFl$.

The above idea of direct computation fails to work for noisy data, because the Kronecker product structure of the covariance matrices breaks down due to the noise. Nonetheless,
 \emph{conjugate gradient methods} can be implemented efficiently in the presence of the KP factorization (\ref{eq:Kmat2KPmat}). We defer the details to Section  \ref{sec:extension}.

\section{Numerical Experiments}
\label{sec:expe}
We first conduct numerical experiments to assess the performance of the proposed algorithm for grid-based designs on test functions in Section \ref{sec:numerical_grid}. Next we employ the proposed method to one-dimensional real datasets in Section \ref{sec:numerical_data} to further assess its performance.

\subsection{Grid-based Designs}\label{sec:numerical_grid}

\subsubsection{Full Grid Designs}\label{sec:numerical_full}
We test our algorithm on the following deterministic function:
\[f(\BFx)=\sin(12\pi x_1)+\sin(12\pi x_2),\quad \BFx\in(0,1)^2.\]
Samples of $f$ are collected from a level-$\eta$ full grid design: $\BFX^{\mathsf{FG}}_\eta=\times_{j=1}^2\{2^{-\eta},2\cdot2^{-\eta},\ldots,1-2^{-\eta}\}$ with $\eta=5,6,\cdots,13$. The proposed KP algorithm is applied for GP regression  using product Mat\'ern correlation function. We choose the same correlation function in each dimension,
with $\omega=1$, and either $\nu=3/2$ or $5/2$. We will investigate the \textit{mean squared error} (MSE) and the average computational time over 1000 random test points for each prediction resulting from KP and the following approximation/fast GP regression algorithms with fixed correlation functions.
\begin{enumerate}
    \item \textbf{laGP} R package \footnote{\url{https://bobby.gramacy.com/r_packages/laGP/}}. In each experiment, laGP is run under  Gaussian covariance family, the only covariance family supported by the package; size of the local subset is set as 100.
    
    \item \textbf{Inducing Points} provided in the \textit{GPML} tool box \citep{JMLR:v11:rasmussen10a}. The number of inducing points $m$ is set as $m=\sqrt{n}$, which is the choice to achieve the optimal approximation power for Mat\'ern-$5/2$ correlation according to \cite{burt2019rates}. However, if the algorithm crashes due to large sample size, $m$ is reduced to a level that the algorithm can run properly. We consider Mat\'ern-$3/2$ and $5/2$ correlations. 
    
    \item \textbf{RFF} to approximate Mat\'ern-$3/2$ and Mat\'ern-$5/2$  correlation functions by feature functions $\bigl[\frac{1}{\sqrt{m}}(\cos \gamma_i x+b_i)\bigr]_{i=1}^m$, where $m=\sqrt{n}$, $\{\gamma_i\}_{i=1}^m$ are independent and identically distributed (i.i.d.) samples from $t$-distributions with degrees of freedom three and five, respectively, and $\{b_i\}_1^m$ are i.i.d. samples from the uniform distribution on $[0,2\pi]$. If the algorithm crashes due to large sample size, $m$ is reduced to a level that the algorithm can run properly.
    
    \item \textbf{Toeplitz} system solver incorporates the one-dimensional Toeplitz method and the Kronercker product technique. We consider  Mat\'ern-$3/2$ and $5/2$ correlations. In this experiment we use equally spaced design points, so that the Toeplitz method can work.
    
\end{enumerate}
 
\begin{figure}[ht]
\centering
\includegraphics[width=0.32\textwidth]{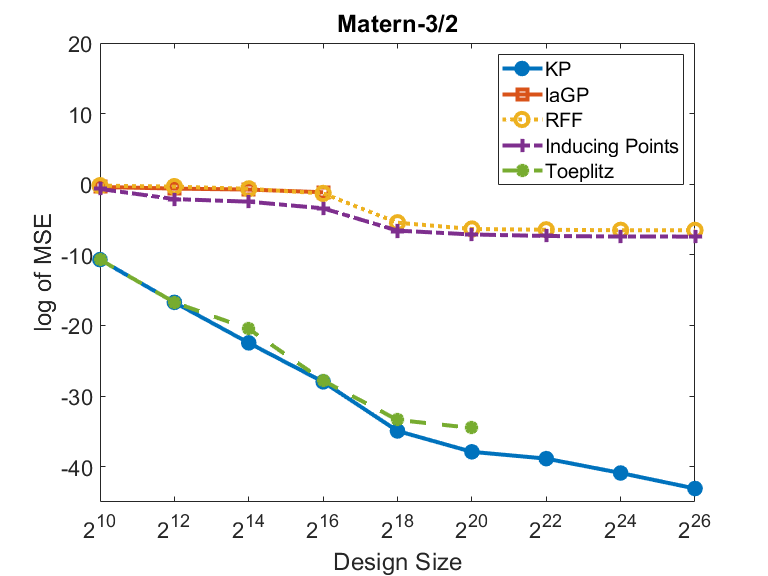}
 \includegraphics[width=0.32\textwidth]{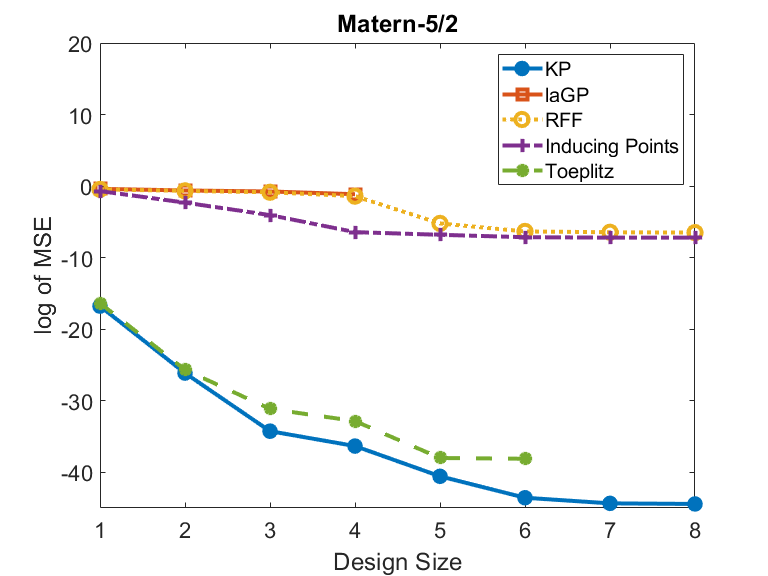}
\includegraphics[width=0.32\textwidth]{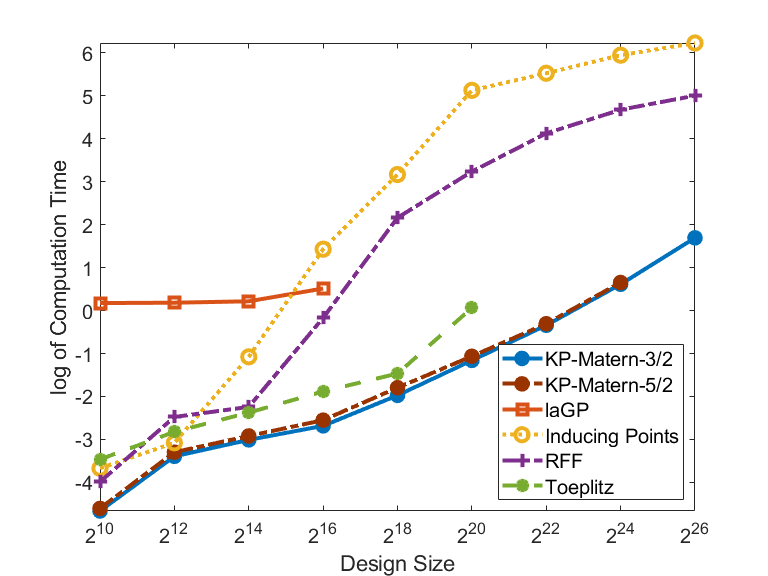}
\caption{Logarithm of MSE for predictions with Mat\'ern-$3/2$ correlation function (left) and Mat\'ern-$5/2$ correlation function (middle) and logarithm of averaged computational time (right)\label{fig:logMSE}. The laGP uses the Gaussian covariance family in both the left and the middle figure. No results are shown for the cases when a runtime error occurs or the prediction error ceases to improve.}
\end{figure}
We sample 1000 i.i.d. test points  uniformly from $(0,1)^2$ for each experimental trial. Figure \ref{fig:logMSE} compares the MSE and the computational time of all algorithms, both under logarithmic scales, for sample sizes $2^{2j}$, $j=5,6,\ldots 13$. 

The performance curves of some algorithms in Figure \ref{fig:logMSE} are incomplete, because these algorithms fail to work at a certain sample size due to a runtime error, or the prediction MSE ceases to improve. In this case, we stop the subsequent experimental trials with larger sample sizes for these algorithms. Specifically, for sample size larger than $2^{16}$, laGP breaks down due to runtime errors. The MSE of Toeplitz ceases to improve at sample size $2^{20}$.  For sample size larger than $2^{20}$, the number of random features for RFF is fixed at $m=2^{10}$ and the number of inducing points for inducing points method is fixed at $m=2^{10}$ for subsequent trials. Otherwise, both the inducing points method and RFF break down  because the approximated covariance matrices are nearly singular. Because of their fixed $m$'s, the performances of RFF and inducing points method do not have noticeable improvement for sample size larger than $2^{18}$. In contrast, KP can run on larger sample sets with sizes up to $2^{26}$ (more than 67 million) grid points.

It is shown in Figure \ref{fig:logMSE} that KP has the lowest MSE and the fastest computational time in all experimental trials. The inducing points method, laGP and RFF have similar MSE in all experimental trials. 
 The Toeplitz and KP algorithms, which compute the GP regression in exact ways, outperform other approximation methods. 

 \subsubsection{Sparse Grid Designs}\label{sec:numerical_sparse}
 We test our algorithm on the \emph{Griewank} function \citep{molga2005test}, defined as
\[f(\BFx)=\sum_{j=1}^d\frac{x_j^2}{4000}-\prod_{j=1}^d\cos\left(\frac{x_j}{\sqrt{j}}\right)+1,\quad \BFx\in(-2,2)^d,\]
with $d=10$ and $d=20$ respectively. Samples of $f$ are collected from a level-$\eta$ sparse grid design $(\eta=3,4,\cdots,7)$. We consider a constant mean $\mu(\BFx)=\beta$ and Mat\'ern correlations with $\nu=3/2,5/2$ and a single scale parameter $\omega$ for all dimensions. We treat mean $\beta$, variance $\sigma^2$ and scale $\omega$ as unknown variables and use the \emph{MLE-predictor}. We compare the performance of proposed KP algorithm and the direct method for GP regression on the sparse grids given in \cite{plumlee2014fast}.

We sample 1000 i.i.d. points uniformly from the input space for each experimental trial. The mean squared error is estimated from these test points. In each trial, the mean squared errors of KP and direct method are in the same order and their differences are within $\pm 10^{-10}$. This is because both methods compute the MLE-predictor in an exact manner, and this also ensures the numerical correctness of the proposed method.

\begin{figure}[ht]
\centering
\includegraphics[width=0.245\textwidth]{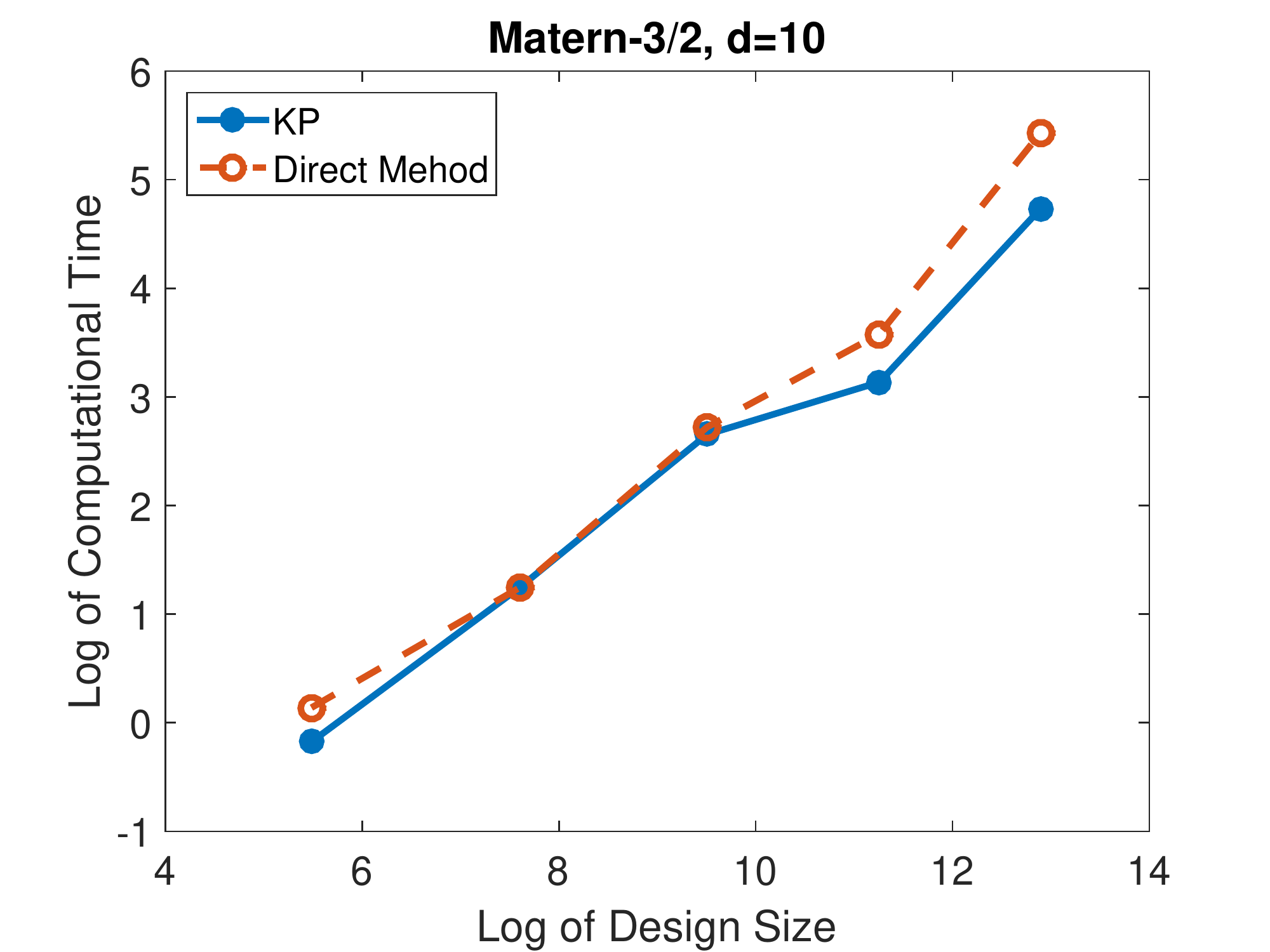}
 \includegraphics[width=0.245\textwidth]{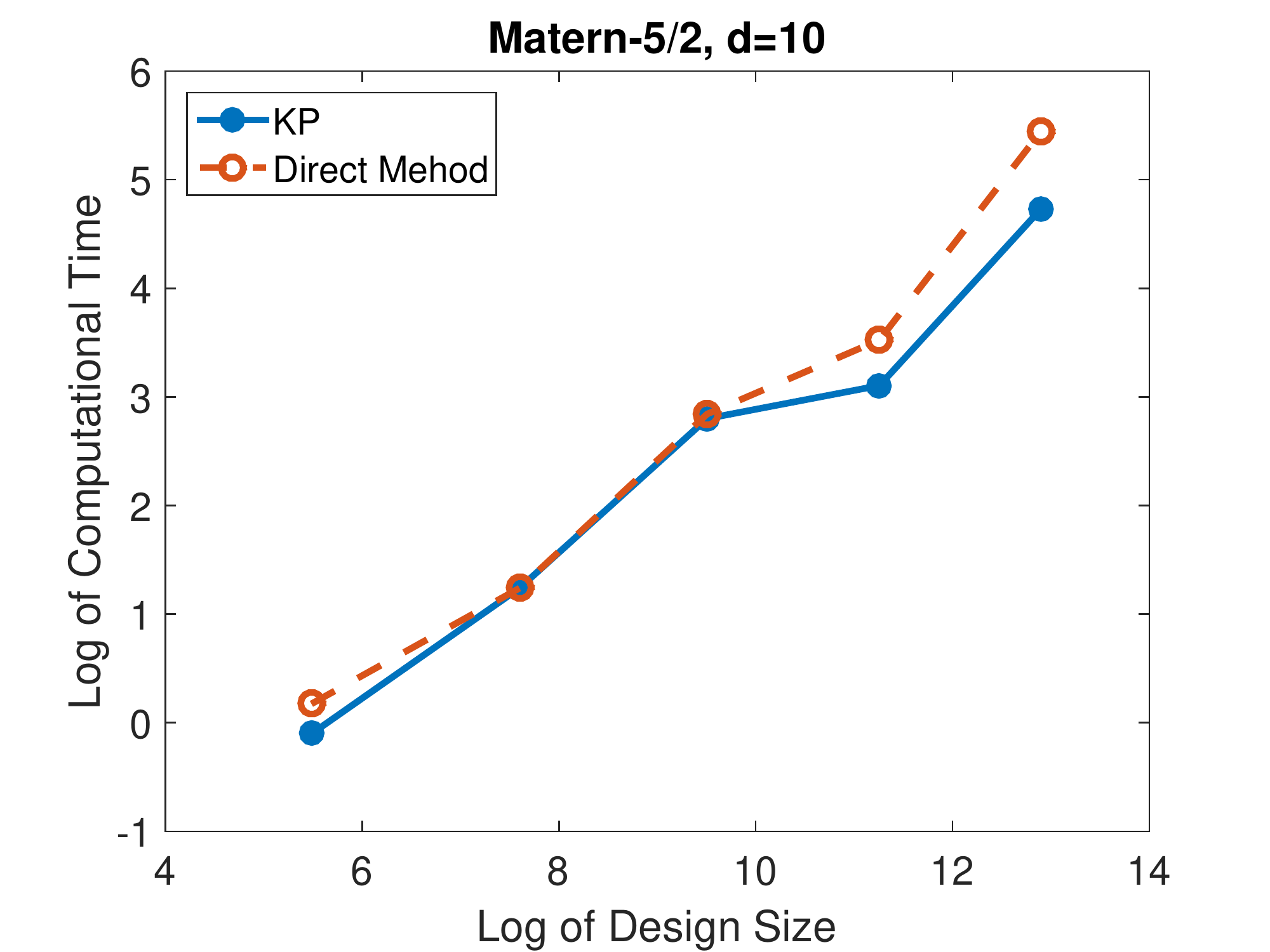}
\includegraphics[width=0.245\textwidth]{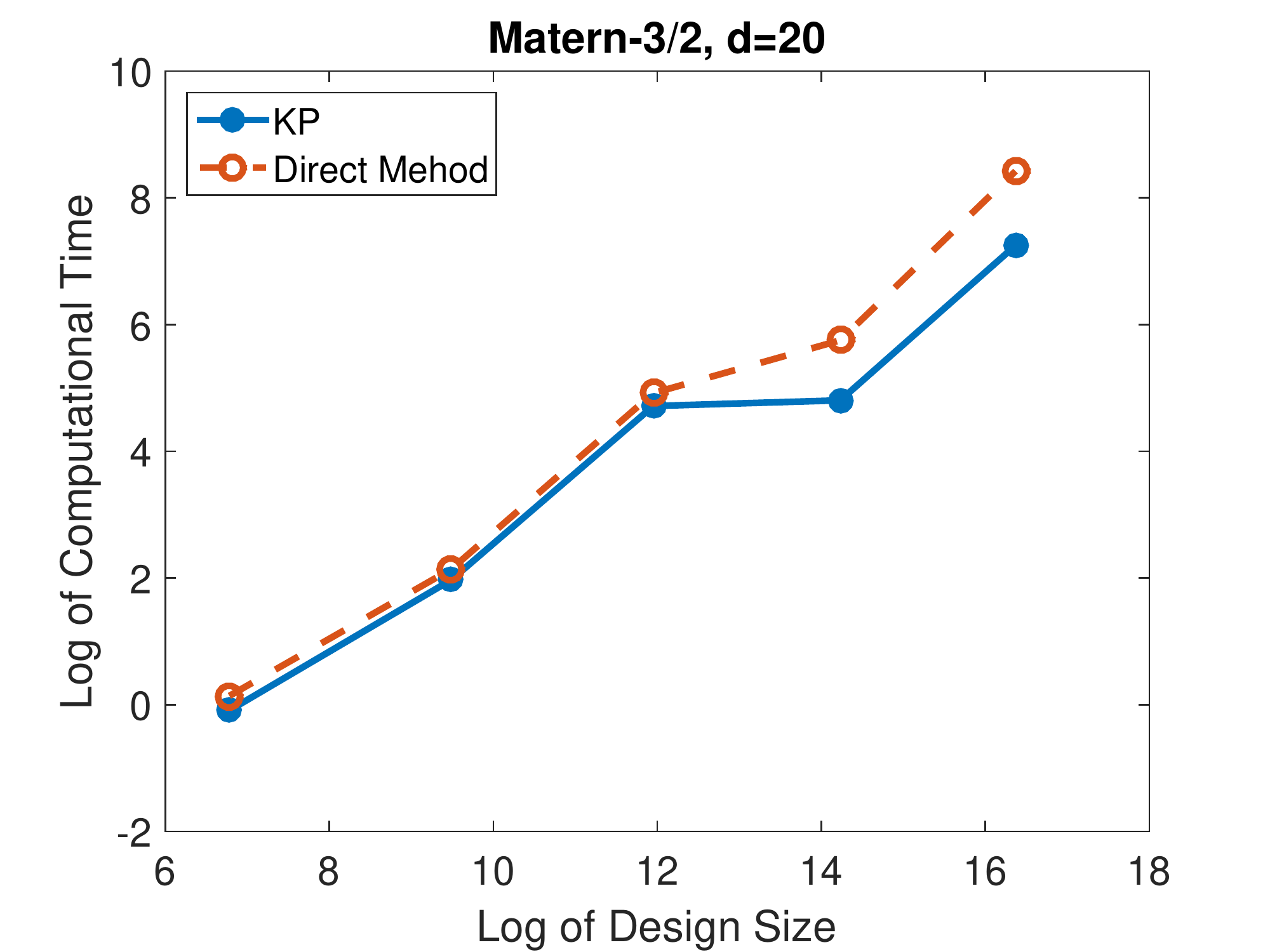}
\includegraphics[width=0.245\textwidth]{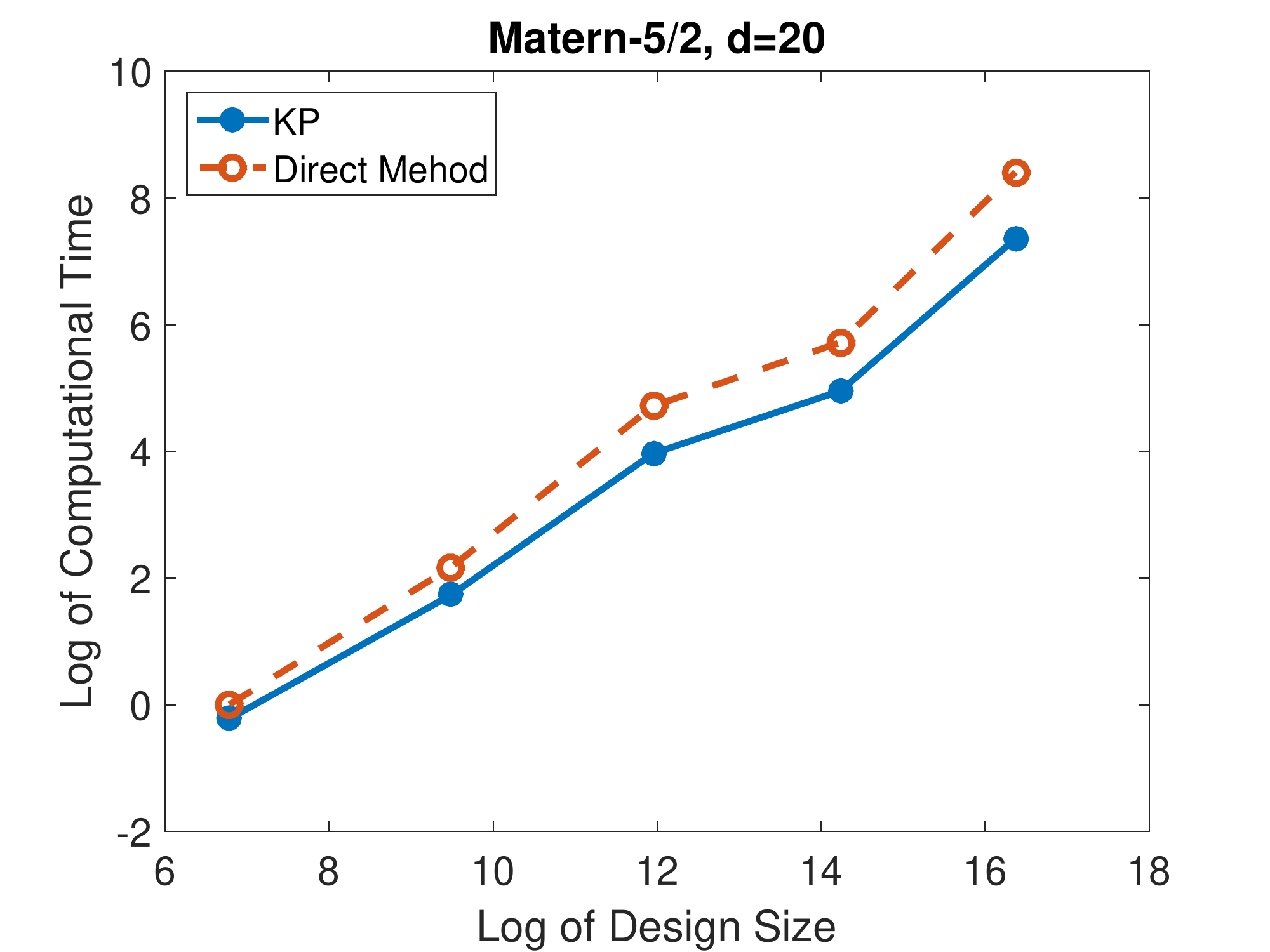}
\caption{Logarithm of the computational time for MLE-predictions with Mat\'ern-$3/2$ correlation function  and Mat\'ern-$5/2$ correlation with $10$ and $20$ dimensional inputs. \label{fig:logTime}}
\end{figure}

Figure \ref{fig:logTime} compares the logarithm of the needed computational time to estimate the unknown parameters $\beta$, $\sigma^2$ and $\omega$ and make predictions on the test points. The proposed KP algorithm is significantly advantageous in the computational time. When there are more than $10^5$ training points, the KP algorithm is at least twice faster than the direct method in each trial. 

\subsection{Real Datasets}\label{sec:numerical_data}
In this section, we assess the performance of the proposed algorithm on two real-world datasets: the Mauna Loa $\text{CO}_2$ dataset \citep{keeling2005atmospheric} and the intraday stock prices of Apple Inc.

\subsubsection{$\text{CO}_2$ Data Interpolation}\label{sec:numerical_data_co2}
This dataset consists of the monthly average atmosphere $\text{CO}_2$ concentrations at the the Mauna Loa Observatory in
Hawaii for the last sixty years. The dataset has in total 767 data points and features a overall upward trend and a yearly cycle. 

We fit the data using GP models reinforced by KP, inducing points, and RFF methods, respectively. For the proposed KP method, we consider a constant mean $\mu(\BFx)=\beta$, a single scale parameter $\omega$, and Mat\'ern correlations with $\nu=3/2,5/2$, respectively. For the inducing points method and RFF, we consider also constant mean $\mu(\BFx)=\beta$. Different from KP, we use Gaussian correlations with scale parameter $\omega$  for the inducing points method and RFF. The number of inducing points is set as 100 for inducing points method. The number of generated random feature is set as 30 for RFF.  We treat mean $\beta$, variance $\sigma^2$ and scale $\omega$ for all algorithms as unknown variables and use the \emph{MLE-predictor}. For each algorithm, we compute the conditional mean and standard deviation on 2000 test points and plot the predictive curve. To evaluate the speed in training and prediction, we record the elapsed times for training and calculate
the average time for a new prediction.

The training and prediction time of each algorithm is shown in Table \ref{tab:time}. It is seen that the KP methods with Mat\'ern-$3/2$ and Mat\'ern-$5/2$ are faster than the inducing points and RFF. The predictive curve given by each algorithm is shown in Figure \ref{fig:co2}. Clearly, both KP methods interpolate adequately from 1960 to 2020 with accurate conditional standard deviations. In contrast, inducing points and RFF fail to interpolate the data, because the numbers of feature functions in inducing points and RFF are less than the number of observations. This results in predictive curves with higher standard deviations.

\begin{table}
\centering
\begin{tabular}{ |p{3cm}|p{2.5cm}|p{2.5cm}|p{2.5cm}|p{2.5cm}|  }
\hline
& \multicolumn{2}{|c|}{$\text{CO}_2$} &\multicolumn{2}{|c|}{Stock Price} \\
\hline
Algorithm & $T_{\text{train}}$\  (sec) & $T_{\text{pred}}$\ ($10^{-3}$ sec)  & $T_{\text{train}}$\ (sec) & $T_{\text{pred}}$\ ($10^{-3}$ sec) \\
\hline
KP Mat\'ern-3/2 & $0.18\pm0.13$  & $2.84\pm 0.96$ & $0.37\pm0.11$ & $2.83\pm 1.07$ \\
\hline
KP Mat\'ern-5/2 & $0.23\pm0.17$  &  $3.31\pm 1.22$ & $0.44\pm0.27$ &   $3.54\pm 1.73$  \\
\hline
Inducing Points &  $0.28\pm0.09$  & $5.26\pm 1.56$& $0.58\pm 0.13 $ & $9.88\pm 3.37$  \\
\hline
RFF &  $0.25\pm0.12$  & $3.34\pm 1.39$& $0.50\pm0.26$& $7.42\pm 0.99$  \\
\hline
\end{tabular}
\caption{Comparisons of training and prediction time \label{tab:time}}
\end{table}

\begin{figure}[ht]
\centering
\includegraphics[width=1\textwidth]{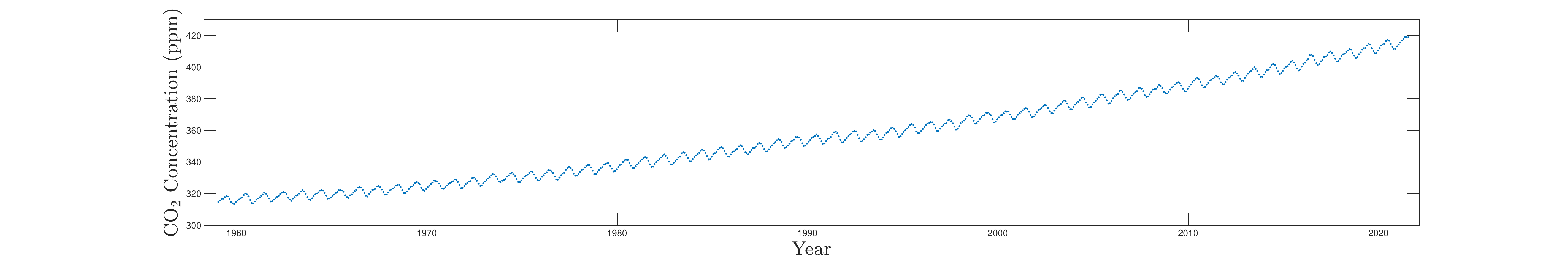}\\
\includegraphics[width=1\textwidth]{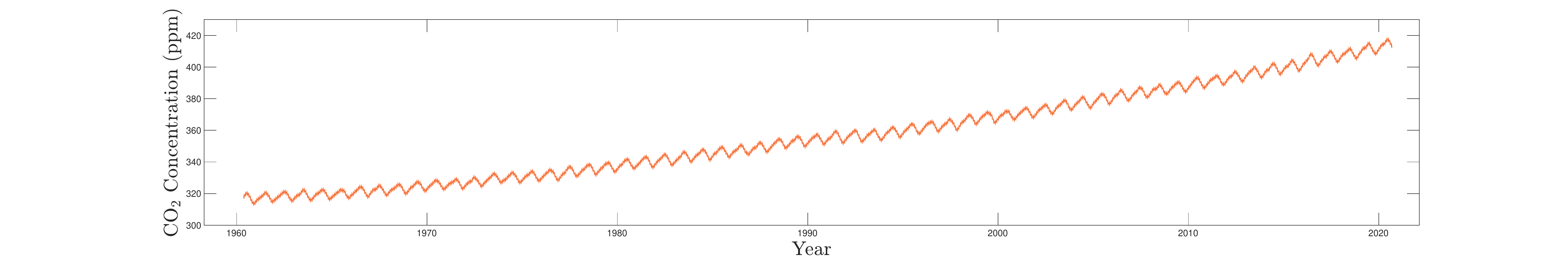}\\
\includegraphics[width=1.\textwidth]{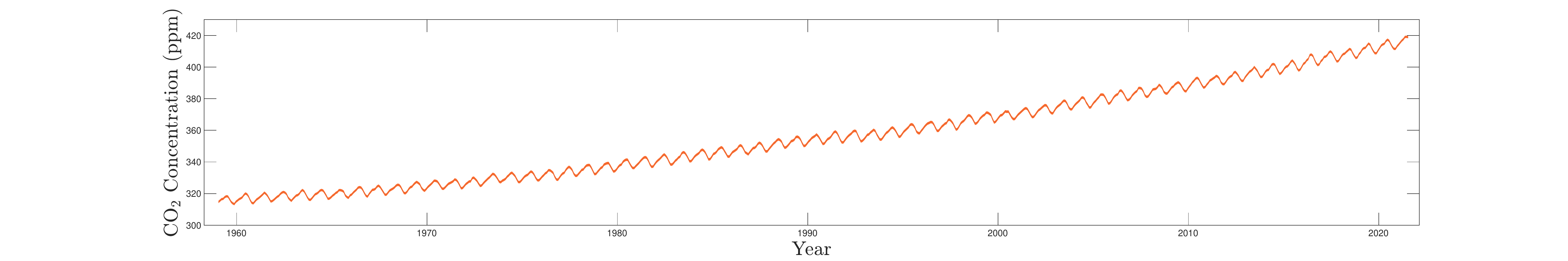}\\
\includegraphics[width=1.\textwidth]{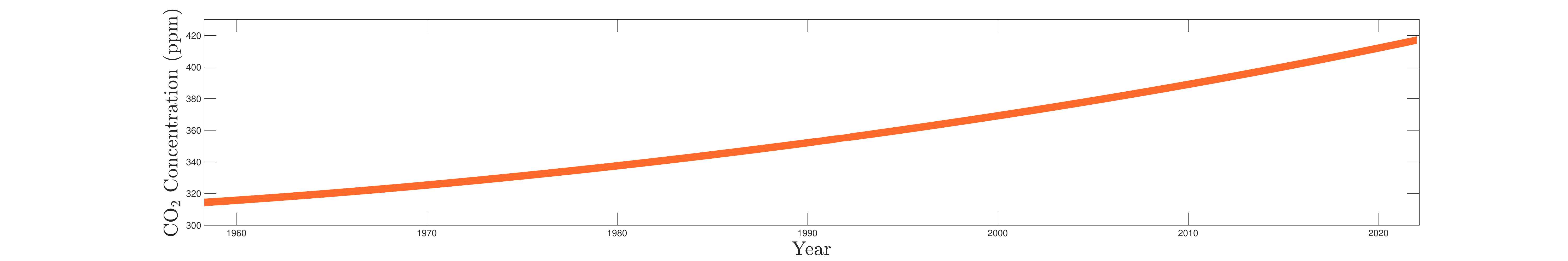}
\includegraphics[width=1.\textwidth]{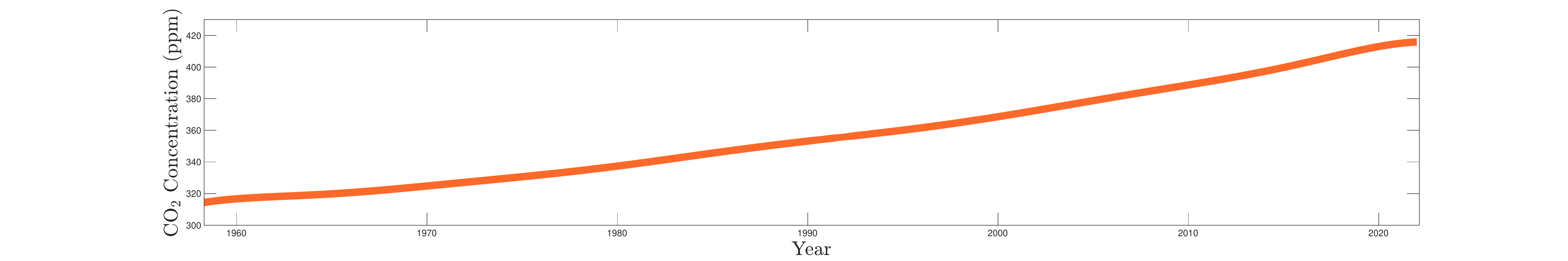}
\caption{Observations of monthly $\text{CO}_2$ concentration (first row) and its interpolation by KP with Mat\'ern-3/2 (second row) , KP with Mat\'ern-5/2 (third row), Inducing Points (fourth row), and RFF (fifth row). The blue curves label predictions and the red areas label $\pm 1$ standard deviation. \label{fig:co2}} 
\end{figure}

\subsubsection{Stock Price Regression}\label{sec:numerical_data_stock}
This dataset consists of the intraday stock prices of Apple Inc from January, 2009 to April, 2011. The dataset has in total 1259 data points. We assume that the data points are corrupted by noise so they are randomly distributed around some underlying trend. In this experiment, our goal is to reconstruct the underlying trend via GP regression.

Similar to Section \ref{sec:numerical_data_co2}, we run KP with Mat\'ern-$3/2$ and Mat\'ern-$5/2$ correlations on the dataset and use inducing pFoints and RFF as our benchmark algorithms. Settings of all algorithms are exactly the same as Section \ref{sec:numerical_data_co2} except that the number of inducing points is set as $200$ and the number of generated random features is set as 100 for RFF. We further treat the data variance parameter $\sigma_Y^2$ as an unknown parameter and use the MLE predictor. We also record the elapsed
times for training and calculate the average time for a new prediction.

\begin{figure}[ht]
\centering
\includegraphics[width=0.45\textwidth]{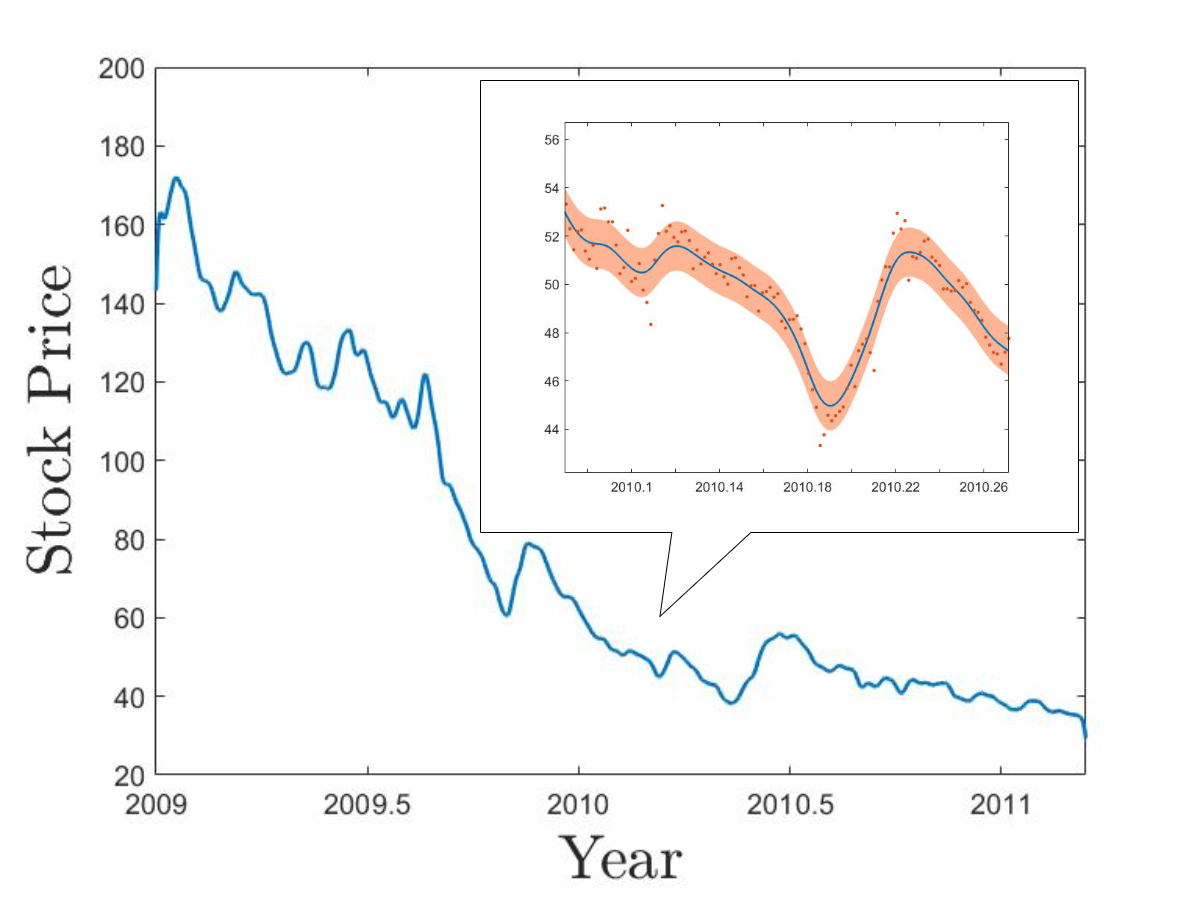}
\includegraphics[width=0.45\textwidth]{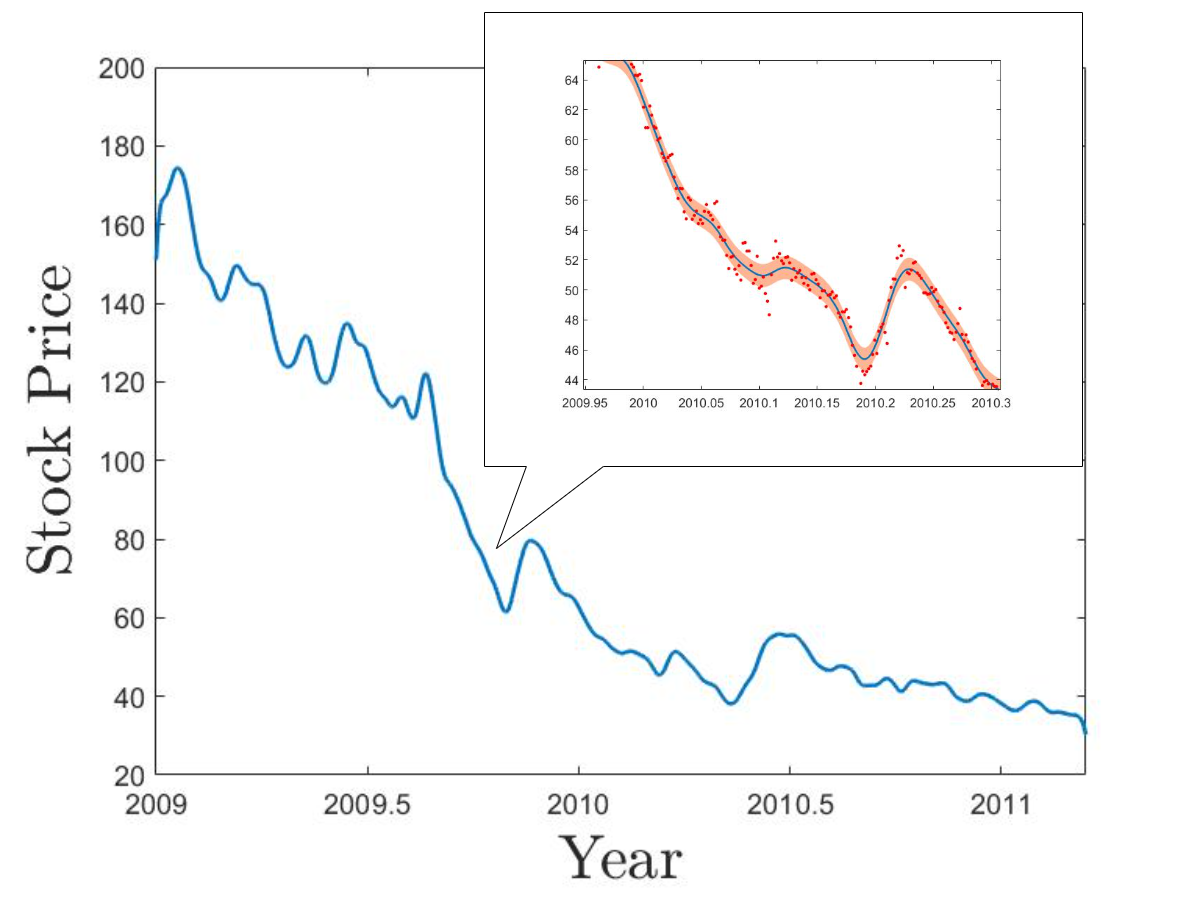}\\
\includegraphics[width=0.45\textwidth]{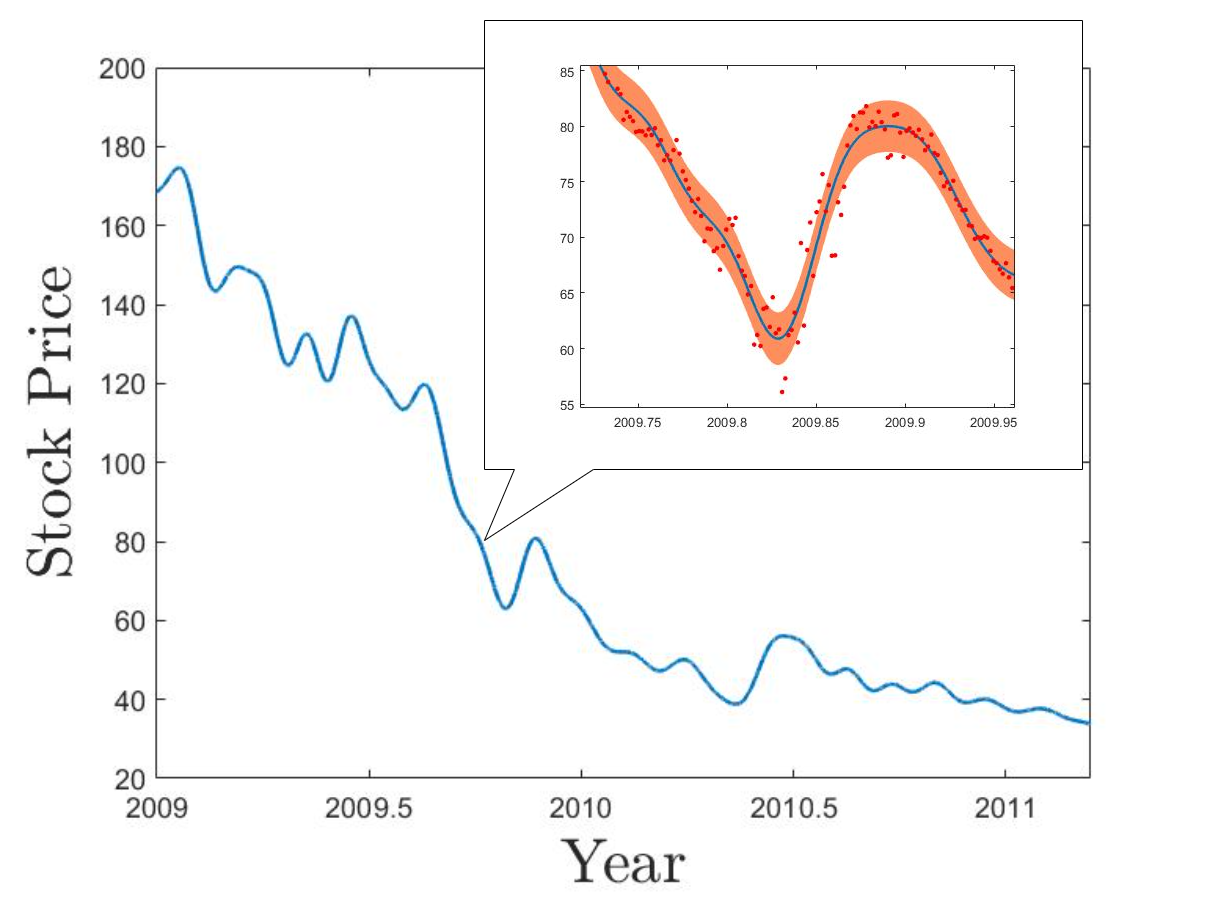}
\includegraphics[width=0.45\textwidth]{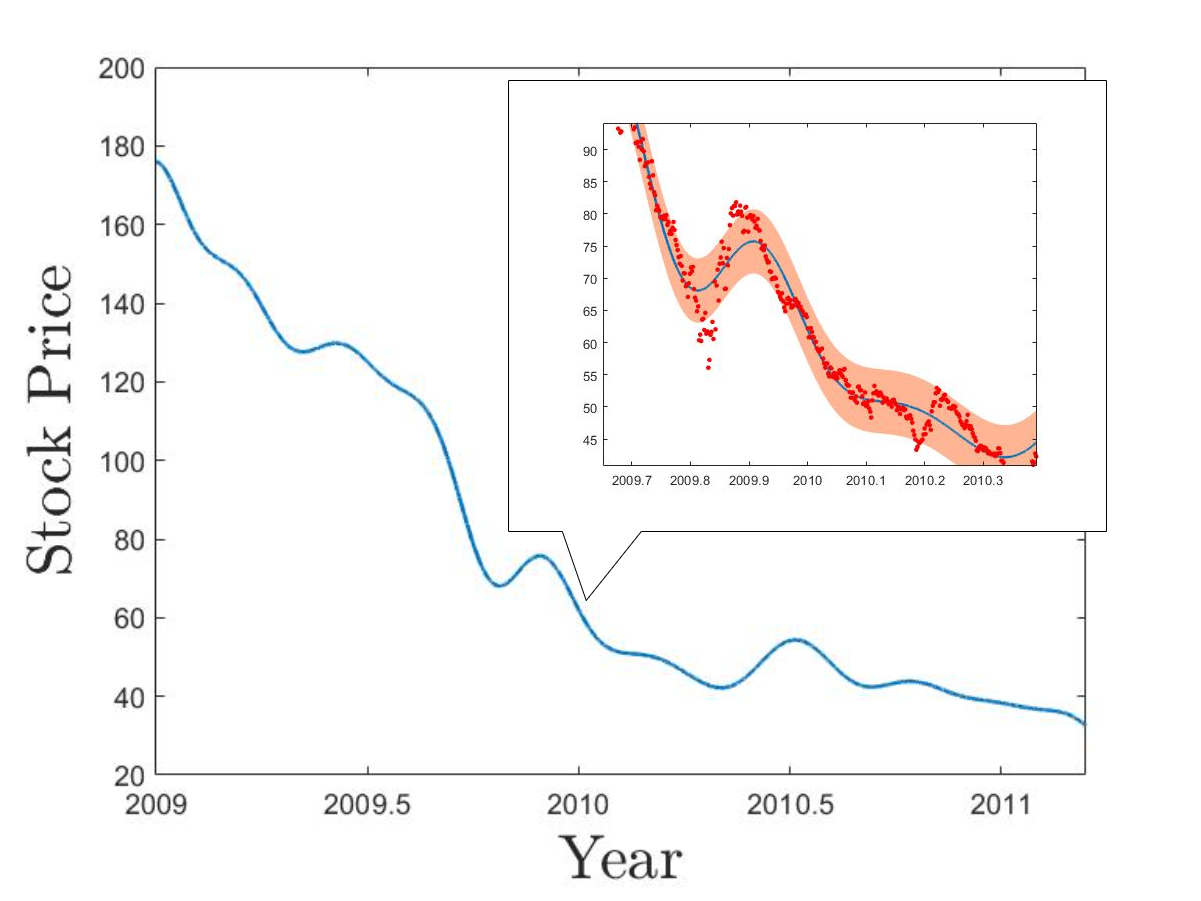}
\caption{Stock Price regression by KP Mat\'ern -3/2 (upper left), KP Mat\'ern -5/2 (upper right), inducing points  (lower left) and RFF (lower right). The blue curves label predictions, the red areas label $\pm 1$ standard deviations, and the red dots labels observations.\label{fig:stock_price}} 
\end{figure}

Similar to the previous experiment, Table \ref{tab:time} shows that the KP methods are more efficient than inducing points and RFF in both training and prediction. The predictive curve given by each algorithm is shown in Figure \ref{fig:stock_price}. We can see that both KP methods successfully capture the local changes of the overall trend while inducing points and RFF fail to do so. This is because neither inducing points nor RFF have enough number of feature functions to reconstruct curves with highly local fluctuations, and therefore, their predictive curves are too smooth so that a larger number of data points are distributed outside of their $\pm1$ standard deviation areas.

\section{Possible Extensions}
\label{sec:extension}
Although the primary focus of this article is on exact algorithms, we would like to mention the potential of combining the proposed method with existing approximate algorithms. In this section, we will briefly discuss how to use the conjugate gradient method in the presence of the KP factorization
(\ref{eq:Kmat2KPmat}) to accommodate a broader class of multi-dimensional Gaussian process regression problems.

\subsection{Multi-dimensional GP Regression with Noisy Data}
\label{sec:mul-dim-noisy}
Suppose the input points lie in a full grid $\BFX^{\rm FG}$  and the observed data $\BFZ$ is noisy:  $Z(\BFx_i)=Y(\BFx_i)+\varepsilon$, where $\varepsilon\sim \mathcal{N}(0,\sigma^2_Y)$. Following arguments similar to what we have done in Sections \ref{sec:algorithm-1D-noisy} and \ref{sec:algorithm-hidg-dim}, we can show that the following matrix operations are essential in computing the posterior and MLE:
\begin{align}
        &\left[\bigotimes_{j=1}^d\BFphi_{j}(\BFX^{(j)})+\frac{\sigma_Y^2}{\sigma^2}\bigotimes_{j=1}^d\BFA_j\right]^{-1}\bold{v}\label{eq:inverse-noisy}\\
        &\log \det\left(\bigotimes_{j=1}^d\BFphi_{j}(\BFX^{(j)})+\frac{\sigma_Y^2}{\sigma^2}\bigotimes_{j=1}^d\BFA_j\right)\label{eq:log-det-noisy}
\end{align}
for some $\bold{v}\in\Real^n$. The direct Kronecker product approach fails to work in this scenario because the additive noise breaks the tensor product structure. Nonetheless, conjugate gradient methods such as those implemented in GPyTorch  \citep{gardner2018gpytorch} or MATLAB \citep{barrett1994templates} can be employed to solve \eqref{eq:inverse-noisy} and \eqref{eq:log-det-noisy}  efficiently. This is because the conjugate gradient methods require nothing more than the multiplication between the covariance matrix and a vector. In our case, both $\{\BFphi_{j}(\BFX^{(j)})\}$ and $\{\BFA_j\}$ are banded matrices so $\bigotimes_{j=1}^d\BFphi_{j}(\BFX^{(j)})$ and $\bigotimes_{j=1}^d\BFA_j$, which are Kronecker products of banded matrices, have only $\CalO(n)$ non-zero entries.
Therefore, the cost of matrix-multiplications by the matres $\bigotimes_{j=1}^d\BFphi_{j}(\BFX^{(j)})$ and $\bigotimes_{j=1}^d\BFA_j$ both scale \emph{linearly} with respect to the number of points in the grid. If input points lie in a sparse grid, the posterior and MLE conditional on noisy data can also be computed efficiently because  they can be decomposed as linear combinations of posteriors and MLEs on full grids as shown in section \ref{sec:algorithm-hidg-dim}.

\subsection{Additive Covariance Functions}
\label{sec:additive-kernel}
Suppose the GP is equipped with the following additive covariance function:
\begin{equation}
    \label{eq:kernel-addivive}
    K(\BFx,\BFx')=\sum_{\CalI\in\CalF}\prod_{j\in\CalI}k_{\CalI,j}(x_j,x'_j)
\end{equation}
where $\CalF$ is any subset of the power set of $\{1,2,\cdots,d\}$ and $k_{\CalI,j}$ is any one-dimensional Mat\'ern correlation with half-integer smoothness. When input points lie in a full grid $\BFX^{\rm FG}$, the posterior and MLE can also be efficiently computed using conjugate gradient methods. In this case, the covariance matrix can be written as the following form:
\begin{equation}
    \label{eq:covariance-additive}
    \BFK=\sum_{\CalI\in\CalF}\left[\bigotimes_{j\in\CalI}\BFphi_{\CalI,j}(\BFX^{(j)})\right]\left[\bigotimes_{j\in\CalI}\BFA_{\CalI,j}^{-1}\right].
\end{equation}
The matrix-multiplication $\BFK\bold{v}$ can be computed in linear time in the size of $\BFX^{\rm FG}$ for any vector $\bold{v}\in\Real^n$.  Firstly, we use KLS techniques introduced in Section \ref{sec:algorithm-hidg-dim} to compute $\bold{v}'=\big[\bigotimes_{j\in\CalI}\BFA^{-1}_{\CalI,j}\big]\bold{v}$, which has linear time complexity in the size of $\BFX^{\rm FG}$. Then, we can compute $\big[\bigotimes_{j\in\CalI}\BFphi_{\CalI,j}(\BFX^{(j)})\big]\bold{v}'$, which has the same time complexity. Similar to Section \ref{sec:mul-dim-noisy}, efficient algorithms also exist when input points lie in a sparse grid.

\section{Conclusions and Discussion}
\label{sec:conc}
In this work, we propose a rapid and exact algorithm for one-dimensional Gaussian process regression under Mat\'ern correlations with half-integer smoothness. The proposed method can be applied to some multi-dimensional problems by using tensor product techniques, including grid and sparse grid designs, and their generalizations \citep{plumlee2021composite}. With a simple modification, the proposed algorithm can also accommodate noisy data. If the design is not grid-based, the proposed algorithm is not applicable. We may apply the idea of \cite{Ding2020} to develop approximated algorithms, which work for not only regression problems, but also for other type of supervised learning tasks.

Another direction for future work is to establish the relationship between KP and the
\textit{state-space approaches.} The latter methods leverage the Gauss-Markov process representation of certain GPs, including Mat\'ern-type GPs with half-integer smoothness, and employ the Kalman filtering and related methodologies to handle GP regression, which results in a learning algorithm with time and space complexity both in $O(n)$ \citep{hartikainen2010kalman,saatcci2012scalable,sarkka2013spatiotemporal,loper2021general}. Whether the key mathematical theories of KP and state-space approaches are essentially equivalent is unknown and requires further investigation. Although having the same time and space complexity as KP, the Kalman filtering method is formulated in a sequential data processing form, which significantly differs from the usual supervised learning framework and makes it more difficult to comprehend. The proposed method, in contrast, is presented by a simple matrix factorization (\ref{eq:Kmat2KPmat}), which is easy to implement and incorporated in more complicated models.


\appendix
\begin{center}
    \LARGE\textbf{Appendix}
\end{center}

\section{Paley-Wiener Theorems}\label{appenA1}

We will need two Paley-Wiener theorems in our proofs, given by Lemmas \ref{lem:Paley-Wiener theorem compact support} and \ref{lem:Paley-Wiener theorem semi-axis}. For detailed discussion, we refer to Chapter 4 of \cite{stein2003complex}. Denote the support of function $f$ as $\operatorname{supp} f$.

\begin{definition}[\cite{stein2003complex}, page 112]
    We say that a function $f$ is of \textbf{moderate decrease} if there exists $M\in\Real$ so that $|f(x)|\leq M/(1+|x|^\alpha)$ for some $\alpha>1$, for all $x\in \Real$.
\end{definition}

\begin{lemma}[Theorem 3.3 in Chapter 4 of \cite{stein2003complex}]\label{lem:Paley-Wiener theorem compact support}
    Suppose $f$ is continuous and of moderate decrease on $\Real$, $\hat{f}$ is the Fourier transform of $f$. Then, $f$ has an extension to the complex plane that is entire with\footnote{\cite{stein2003complex} uses an equivalent but different definition of the inverse Fourier transform as $\tilde{f}(\xi)=\int_{-\infty}^\infty f(x)e^{2\pi i \xi} d x$, so that this inequality becomes $|f(z)| \leq Ae^{2\pi M|z|}$.} $|f(z)| \leq Ae^{M|z|}$ for some $A>0$, if and only if $\operatorname{supp}\hat{f}\subset[-M,M]$.
\end{lemma}

\begin{lemma}[Theorem 3.5 in Chapter 4 of \cite{stein2003complex}]\label{lem:Paley-Wiener theorem semi-axis}
    Suppose $f$ is continuous and of moderate decrease on $\Real$, $\hat{f}$ is the Fourier transform of $f$. Then $\operatorname{supp}\hat{f}\subset[0,+\infty)$ if and only if $f$ can be extended to a continuous and bounded function in the closed upper half-plane $\{ z=x+iy:y\geq 0 \}$ with $f$ holomorphic in the interior.   
\end{lemma}

\section{Technical Proofs}\label{appenA2}

\subsection{Algebraic Properties}

The following Lemma \ref{lem:roots} will be useful in proving the main theorems. We use $\deg p$ to denote the degree of polynomial $p$. For notational convenience, we define the degree of the zero polynomial as $-1$. We say $x$ a \textit{zero} of function $f$ if $f(x)=0$.

\begin{lemma}\label{lem:roots}
    Let $p_1$ and $p_2$ be polynomials with $\deg p_1=d_1,\deg p_2=d_2$. If $\max(\deg p_1,\deg p_2)\geq 0$ and $c\neq 0$, then the function
    $f(x):=p_1(x)e^{cx}+p_2(x)e^{-cx} $
    has at most $d_1+d_2+1$ real-valued zeros.
\end{lemma}

\begin{proof}
Without loss of generality, we assume that $p_1$ is non-zero. Suppose $f$ has at least $d_1+d_2+2$ real-valued zeros. Equivalently, the function
	$g(x)=p_1(x)e^{2cx}+p_2(x) $
	has at least $d_1+d_2+2$ real-valued zeros.  The mean value theorem implies $g'(\xi)=0$ for some $\xi$ lying between two consecutive real-valued zeros of $g$. Therefore, $g'$ has at least $d_1+d_2+1$ real-valued zeros. Repeating this procedure $d_2+1$ times, we can conclude that $g^{(d_2+1)}$ has at least $d_1+1$ real-valued zeros. Note that $g^{(d_2+1)}$ possesses the form
	$g^{(d_2+1)}(x)=q(x)e^{2cx}, $
	where $q(x)$ is a non-zero polynomial with degree $d_1$. Because $e^{2cx}>0$, $q(x)$ has at least $d_1+1$ real-valued zeros, which contradicts the fundamental theorem of algebra.
\end{proof}

Lemma \ref{lem:trivialsolution} can directly lead to Theorems \ref{theo:solutionspace} and \ref{theo:solutionspace_one_sided}. We call the vector of zero the \textit{trivial solution} to a homogeneous system of linear equations.

\begin{lemma}\label{lem:trivialsolution}
The following homogeneous systems have only the trivial solutions.
\begin{enumerate}
    \item For $m\geq 1$, the $m\times m$ system about $(u_1,\ldots,u_{m})^T$:	
    \begin{eqnarray*}
		\sum_{j=1}^{m} b_j^l \exp\{c b_j\} u_j=0,
	\end{eqnarray*} 
	with $l=0,\ldots,m-1$, $c\neq 0$ and distinct real numbers $b_1,\ldots,b_{m}$.
	\item For $m,s\geq 1$, the $(m+s)\times (m+s)$ system about $(u_1,\ldots,u_{m+s})^T$:
	\begin{eqnarray}\label{homosys2}
		\sum_{j=1}^{m+s} b_j^l \exp\{c b_j\} u_j=0, \quad \sum_{j=1}^{m+s} b_j^r \exp\{-c b_j\} u_j=0,
	\end{eqnarray} with $l=0,\ldots,m-1$, $r=0,\ldots,s-1$, $c\neq 0$ and distinct real numbers $b_1,\ldots,b_{m+s}$.
\end{enumerate}
\end{lemma}

\begin{proof}
For both parts, it suffices to prove that the coefficient matrices are of full row ranks, which is equivalent to that they are of full column ranks. This inspires us to consider the transpose of the coefficient matrices.


Here we only provide the proof for Part 2. The proof for Part 1 follows from similar lines.
For Part 2, the linear system corresponding to the transposed coefficient matrix is
\begin{eqnarray}\label{homosys4}
\sum_{l=0}^{m-1}b_j^l \exp\{c b_j\} v_l+\sum_{r=0}^{s-1}b_j^r \exp\{-c b_j\} v_{m+r}=0,
\end{eqnarray}
with the vector of unknowns $(v_0,\ldots,v_{m-1})^T$. Suppose (\ref{homosys2}) has a non-trivial solution. Then (\ref{homosys4}) also has a nontrivial solution, denoted as $(v^*_0,\ldots,v^*_{m+s-1})^T$. Write $p_1(x)=\sum_{l=0}^{m-1} v^*_lx^l$ and $p_2(x)=\sum_{r=0}^{s-1}v^*_{m+r}x^r$. Therefore, (\ref{homosys4}) implies that each $b_j$ is a zero of the function $f(x):=p_1(x)e^{cx}+p_2(x)e^{-cx}$. Hence $f(x)$ has at least $m+s$ distinct zeros. Note that $\deg p_1\leq m-1,\deg p_2\leq s_1$. Because $(v^*_0,\ldots,v^*_{m+s-1})^T$ is non-trivial, we have $\max(\deg p_1,\deg p_2)\geq 0$. Thus Lemma \ref{lem:roots} yields that $f(x)$ has no more than $m+s-1$ distinct zeros, a contradiction.
\end{proof}

\begin{proof}[Proof of Theorem \ref{theo:solutionspace}]
This theorem follows directly from Part 2 of Lemma \ref{lem:trivialsolution}, because each $(k-1)\times (k-1)$ submatrix of the coefficient matrix corresponds a linear system of the form in Part 2 of Lemma \ref{lem:trivialsolution}.
\end{proof}

\begin{proof}[Proof of Theorem \ref{theo:solutionspace_one_sided}]
This theorem follows directly from Lemma \ref{lem:trivialsolution}, because each $(s-1)\times (s-1)$ submatrix of the coefficient matrix corresponds a linear system of the form in of Lemma \ref{lem:trivialsolution}.
\end{proof}

\begin{proof}[Proof of Theorem \ref{theo:invariant}]
Let $(A_1,\ldots,A_k)^T$ be a solution to (\ref{eq:LE}). It suffices to prove that
$$\sum_{j=1}^k A_j (a_j+t)^l \exp\{\delta c (a_j+t)\}=0, $$
for $l=0,\ldots,(k-3)/2$, $\delta=\pm 1$ and each $t\in\mathbb{R}$. This can be proved by noting that
\begin{eqnarray*}
&&\sum_{j=1}^k A_j (a_j+t)^l \exp\{\delta c (a_j+t)\}\\
&=&\exp\{\delta c t\}\sum_{j=1}^k A_j \exp\{\delta c a_j\} \sum_{m=0}^l \binom{l}{m} a_j^m t^{l-m}\\
&=&\exp\{\delta c t\}\sum_{m=0}^l \binom{l}{m} t^{l-m} \left(\sum_{j=1}^k A_j a_j^m \exp\{\delta c a_j\}\right)=0,
\end{eqnarray*}
where the last equality follows from the identity $\sum_{j=1}^k A_j a_j^m \exp\{\delta c a_j\}=0$ for $0\leq m\leq l$, ensured by equation system (\ref{eq:LE}).
\end{proof}

\begin{proof}[Proof of Theorem \ref{theo:shift_one_sided}]
The proof follows from arguments similar to the proof of Theorem \ref{theo:invariant}.
\end{proof}

\subsection{Results for the Supports}

We first prove the following useful lemma.

\begin{lemma}\label{lem:connected}
    Let $K$ be a Mat\'ern correlation with a half-integer smoothness. Suppose $b_1<\cdots<b_m$ and $t\in (b_\tau,b_{\tau+1})$ for some $1\leq \tau<n$. Let
    $\psi(x)=\sum_{j=1}^m B_j K(x,b_j), $
    for $B_j\in\mathbb{R}$. Denote
    $\psi_1(x)=\sum_{j=1}^\tau B_j K(x,b_j)$, and $\psi_2(x)=\sum_{j=\tau+1}^m B_j K(x,b_j)$.
    If there exists $\epsilon>0$ such that for $x\in(t-\epsilon,t+\epsilon)$, $\psi(x)=0$, then
    \[
    \psi_1(x)=
    \begin{cases}
    \psi(x), & \text{for } x< b_{\tau},\\
    0, & \text{otherwise}.
    \end{cases}
    \text{~~~and~~~}
    \psi_2(x)=
    \begin{cases}
    0, & \text{for } x\leq b_{\tau+1},\\
    \psi(x), & \text{otherwise}.
    \end{cases}
    \]
\end{lemma}

\begin{proof}
It is known that when $\nu=p+1/2$ with $p\in\mathbb{N}$, the Mat\'ern correlation can be expressed as \citep{santner2003design}
\begin{eqnarray}\label{maternpolynomial}
K(x,x')=P_p(|x-x'|)\exp\{-c|x-x'|\},
\end{eqnarray}
where $c=\sqrt{2\nu/\omega}$ and $P_p(x)=\sigma^2\frac{p!}{(2p)!}\sum_{j=0}^p\frac{(p+j)!}{j!(p-j)!}(2cx)^{p-j}$ is a polynomial of degree $p$.

Therefore, for any $x\in (b_\tau,b_{\tau+1})$,
\begin{eqnarray*}
\psi(x)&=&\psi_1(x)+\psi_2(x)\\
&=&\sum_{j=1}^\tau B_j P_p(x-b_j)e^{-c(x-b_j)} +\sum_{j=\tau+1}^m B_jP_p(b_j-x)e^{c(x-b_j)}\\
&=:&p_1(x)e^{-cx}+p_2(x)e^{cx},
\end{eqnarray*}
where $p_1$ and $p_2$ are polynomials. Thus $\psi(x)$ is an analytic function on $(b_\tau,b_{\tau+1})$. Then $\psi(x)=0$ for $x\in(t-\epsilon,t+\epsilon)$ implies $\psi(x)=0$ for $x\in(b_\tau,b_{\tau+1})$, which is possible only if $p_1\equiv p_2\equiv 0$, because otherwise $\psi$ can only have at most $2p+1$ distinct zeros on $(b_\tau,b_{\tau+1})$ according to Lemma \ref{lem:roots}. Hence, $\psi_1(x)=0$ whenever $x\geq b_\tau$, and $\psi_2(x)=0$ whenever $x\leq b_{\tau+1}$.
\end{proof}

The following lemma formalizes the rationale behind (\ref{eq:LE}).

\begin{lemma}\label{lem:entire}
Let $U$ be a connected open subset of $\mathbb{C}$ containing a point $z_0$.
    Let $f(z)$ be a holomorphic function on $U$, and $m$ a positive integer. Then $f(z)(z-z_0)^{-m}$ can be extended as a holomorphic function on $U$ if and only if
	$f^{(j)}(z_0)=0,$ 	
for $j=0,\ldots,m-1$,
\end{lemma}

\begin{proof}
First assume $f(z)(z-z_0)^{-m}$ being holomorphic. Then $f(z_0)$ must be zero, because otherwise $\lim_{z\rightarrow z_0}f(z)(z-z_0)^{-m}=\infty$. If $f$ vanishes identically in $U$, the desired result is trivial. If $f$ does not vanish identically in $U$, according to Theorem 1.1 in Chapter 3 of \cite{stein2003complex}, there exists a neighborhood $V\subset U$ of $z_0$, and a unique positive integer $m'$ such that $f(z)=(z-z_0)^{m'} g(z)$ for $z\in V$ with $g$ being a non-vanishing holomorphic function on $V$. Clearly it must hold that $m'\geq m$, because otherwise we have $\lim_{z\rightarrow z_0}f(z)(z-z_0)^{-m}=\infty$ again. Then it is easily checked that $f^{(j)}(z_0)=0,$ for $j=0,\ldots,m-1$.

For the converse, in a small disc centered at $z_0$ the function $f$ has a power series expansion $f=\sum_{j=0}^\infty a_j(z-z_0)^j$, where $a_j=f^{(j)}(z_0)/j!$ for each $j\in\mathbb{N}$. Thus $a_0=\cdots=a_{m-1}=0$. Consequently, $f(z)(z-z_0)^{-m}=\sum_{j=m}^\infty a_j(z-z_0)^{j-m}$ and thus is holomophic on $U$.
\end{proof}

\begin{proof}[Proof of Theorem \ref{theo:matern_KPdegree}]
As before, let $k:=2\nu+2$. Suppose that $\phi(x)=\sum_{j=1}^m A_j K(x,a_j)$ has a compact support. The analytic continuation of its inverse Fourier transform is
$$\tilde{\phi}(z)=\sum_{j=1}^m A_j \exp\{a_j z\} (c^2+z^2)^{(k-1)/2}:=\gamma(z)(c^2+z^2)^{(k-1)/2},$$
for $z\in\mathbb{C}\in \{\pm ci\}$. Then Lemma \ref{lem:entire} entails $\gamma^{(j)}(\pm ci)=0$ for $j=0,\ldots,(k-3)/2$, which leads to the linear system
$$\sum_{j=1}^m A_j a_j^l\exp\{\delta c a_j\}=0, $$
with $l=0,\ldots,(k-3)/2$ and $\delta=\pm 1$. But this system has only the trivial solution in view of Lemma \ref{lem:trivialsolution}.
\end{proof}

\begin{proof}[Proof of Theorem \ref{theo:support1}]Without loss of generality, we can assume that $a_1=-M$ and $a_k=M$ for some positive real number $M$, because otherwise we can apply a shift translation to convert the original problem to this form in view of Theorem \ref{theo:invariant}.

We first employ Lemma \ref{lem:Paley-Wiener theorem compact support} to show that $\operatorname{supp}\phi_\mathbf{a}\subset[-M,M]=[a_1,a_k]$.
Lemma \ref{lem:entire} implies that $\tilde{\phi}_\mathbf{a}$ is entire. By its continuity, $|\tilde{\phi}_\mathbf{a}|$ is bounded in the region $|z|\leq 2c$. For $|z|\geq 2c$, we have
\begin{align*}
    \left|\tilde{\phi}_\mathbf{a}(z)\right|={}& \left|\gamma(z)\right|\cdot \left|(c^2+z^2)^{(k-1)/2}\right|\leq c^{k-1}\left|\gamma(z)\right|\\
    \leq{} & c^{k-1}\sum_{j=1}^k\left|A_j\right| \cdot \left|\exp\{-ia_j z\}\right|
    \leq c^{k-1}\sum_{j=1}^k\left|A_j\right| \exp\{M |z|\},
\end{align*}
 where the last inequality follows from the fact that $|e^{z}|\leq e^{|z|}.$ Clearly, $\tilde{\phi}_{\mathbf{a}}$ is of moderate decrease. According to Lemma \ref{lem:Paley-Wiener theorem compact support}, we obtain that $\operatorname{supp}\phi_\mathbf{a}\subset [-M,M]$.  
 
It remains to prove that $\operatorname{supp}\phi_\mathbf{a}= [-M,M]$. Suppose $\operatorname{supp}\phi_\mathbf{a}\neq [-M,M]$. Then we can find $-M\leq M_1<M_2\leq M$, such that $\phi_{\mathbf{a}}(x)=0$ for $x\in[M_1,M_2]$. Therefore, Lemma \ref{lem:connected} implies that $\phi_{\mathbf{a}}$ can be expressed as
$$\phi_{\mathbf{a}}(x)=\sum_{j=1}^\tau A_j K(x,a_j)+\sum_{j=\tau+1}^k A_j K(x,a_j):=\psi_1(x)+\psi_2(x),$$
for some $1\leq \tau< n$, such that $\operatorname{supp}\psi_1\subset [a_1,a_{\tau}],\operatorname{supp}\psi_2\subset [a_{\tau+1},a_n]$. Because either $\psi_1$ or $\psi_2$ must not be identically vanishing, such a function, according to Definition \ref{Def:KP}, is a KP with degree less than $k$. But this contradicts Theorem \ref{theo:matern_KPdegree}.
\end{proof}

\begin{proof}[Proof of Theorem \ref{theo:matern_noKP}]
For Part 1, when the smoothness parameter $\nu$ of a Mat\'ern kernel is not a half integer, then direct calculations shows
\begin{eqnarray}
\tilde{\phi}_a(x)&\propto& \left [ \sum_{j=1}^{k} A_j \exp\{i a_j x\} \right](c^2+x^2)^{-\nu-\frac{1}{2}}\nonumber\\
&=&\left [ \sum_{j=1}^{k} A_j \exp\{i a_j x\} \right]\exp\left\{-(\nu+\frac{1}{2})\log(x^2+c^2)\right\}.\label{fractionalmatern}
\end{eqnarray}
The goal is to prove that (\ref{fractionalmatern}) cannot be extended to an entire function unless $A_j=0$ for each $j$.
There is no continuous complex logarithm function defined on all $\mathbb{C}\setminus\{0\}$. Here we consider the principal branch of the complex logarithm
$\operatorname{Log} z:=\log |z|+i\operatorname{Arg} z, $
where $\operatorname{Arg} z$ is the principal value of the argument of $z$ ranging in $(-\pi,\pi]$. For $x\in\mathbb{R}$, we have $\log x=\operatorname{Log} x$. It is known that $\operatorname{Log} x$ is holomorphic on the set $\mathbb{C}\setminus \{z\in\mathbb{R}:z\leq 0\}$. Therefore, the function in (\ref{fractionalmatern}) can be analytically continued to the region $\mathcal{S}:=\mathbb{C}\setminus\mathcal{S}^c:=\mathbb{C}\setminus\{yi:y\in\mathbb{R},|y|\geq c\}.$

Because analytical continuation is unique, the analytical continuation of (\ref{fractionalmatern}) should coincide with
$$g(z):=\left [ \sum_{j=1}^{k} A_j \exp\{i a_j z\} \right]\exp\left\{-(\nu+\frac{1}{2})\operatorname{Log}(z^2+c^2)\right\} $$
on $\mathcal{S}$. Because $\nu+1/2$ is not an integer, $\exp\left\{-(\nu+\frac{1}{2})\operatorname{Log}(x^2+c^2)\right\}$ is discontinuous when $z$ moves across $\mathcal{S}^c$. Suppose there exists a KP. Then $g(z)$ must an entire function in view of lemma \ref{lem:Paley-Wiener theorem compact support}. To make $g(z)$ continuous on $\mathcal{S}^c$, we must have $\sum_{j=1}^{k} A_j \exp\{i a_j z\}=0$ on $\mathcal{S}^c$, but this readily implies $\sum_{j=1}^{k} A_j \exp\{i a_j z\}=0$ on $\mathbb{C}$ as $\sum_{j=1}^{k} A_j \exp\{i a_j z\}$ is also an entire function. Therefore $g(z)=0$. By the uniqueness of the Fourier transform, the underlying KP vanishes identically in $\mathbb{R}$, which leads to a contradiction.

For part 2, a Gaussian correlation function $K$ is an analytic function on $\mathbb{R}$, and so does $\phi_{\mathbf{a}}:=\sum_{j=1}^n A_j K(x-a_i)$. Therefore, $\phi_{\mathbf{a}}$ cannot have a compact support unless $\phi_{\mathbf{a}}\equiv 0$.
\end{proof}

\begin{proof}[Proof of Theorem \ref{theo:support1_one_sided}]
Without loss of generality, we can assume that $a_1=0$ because otherwise we can apply shift translation to make this happen in view of Theorem \ref{theo:shift_one_sided}.

Clearly, $\phi_{\BFa}$ is of moderate decrease in view of the expression (\ref{maternpolynomial}).
Direct calculation shows 
\[\tilde{\phi}_\BFa(z)\propto \left [ \sum_{j=1}^{s} A_j \exp\{i a_j z\} \right](c^2+z^2)^{-(k-1)/2}=\gamma(z)(c^2+z^2)^{-(k-1)/2}.\]
Equation (\ref{eq:LE_one_sided_R}) implies $\gamma^{(j)}(ci)=0$ for $j=0,\ldots,(k-3)/2$. Thus $\frac{d^j}{dz^{j}}(\gamma(z)(z+ci)^{-(k-1)/2})\big|_{z=ci}=0$ for $j=0,\ldots,(k-3)/2$, which, together with Lemma \ref{lem:entire}, yields that $f(z):=\gamma(z)(c^2+z^2)^{-(k-1)/2}$ is holomorphic in a neighborhood of $ci$. So $f(z)$ is continuous on the upper half-plane $\{z=x+iy:y\geq 0\}$ and is holomorphic in its interior. To employ Lemma \ref{lem:Paley-Wiener theorem semi-axis}, it remains to proof that $f(z)$ is bounded in $\{z=x+iy:y\geq 0\}$. For $|z-ci|\leq c$, $f(z)$ is clearly bounded as it is a continuous function. For $|z-ci|\geq c$ and $z\in\{z=x+iy:y\geq 0\}$, we have
$$|(c^2+z^2)^{-(k-1)/2}|=|z-ci|^{-(k-1)/2}|z+ci|^{-(k-1)/2}\leq c^{-(k-1)}. $$
Write $z=x+iy$, then
\begin{eqnarray*}
&&|f(z)|=|\gamma(z)||(c^2+z^2)^{-(k-1)/2}|
\leq \left|\sum_{j=1}^{s} A_j \exp\{i a_j (x+iy)\}c^{-(k-1)}\right|\\
&\leq& c^{-(k-1)}\sum_{j=1}^s \left|A_j\right|\left|\exp\{i a_j (x+iy)\}\right|
=c^{-(k-1)}\sum_{j=1}^s \left|A_j\right|\exp\{- a_j y\},
\end{eqnarray*}
which is bounded as $y\geq 0$. Therefore, according to Lemma \ref{lem:Paley-Wiener theorem semi-axis},  $\operatorname{supp}\phi_{\BFa}\subset [0,+\infty)$.

Next, we prove that $0\in\operatorname{supp}\phi_{\BFa}$. 
First, it can be shown that $A_1\neq 0$, because otherwise $(A_2,\ldots,A_s)^T$ is a solution to the linear system
$\sum_{j=2}^s a_j^l \exp\{-c a_j\}A_j=0,$
with $l=0,\ldots,(k-3)/2$, if $s=(k+1)/2$, or
$\sum_{j=2}^s a_j^l \exp\{-c a_j\}A_j=0, \quad \sum_{j=2}^s a_j^r \exp\{c a_j\}A_j=0,$
with $l=0,\ldots,(k-3)/2$ and $r=0,\ldots,s-(k+3)/2$, if $s\geq (k+3)/2$. Then Lemma \ref{lem:roots} suggests that $A_j=0$ for all $j=0,\ldots,s$, which is a contradiction.
Now, suppose $0\not\in \operatorname{supp}\phi_{\BFa}$. Then there exists $\epsilon>0$, such that $\phi_{\BFa}(x)=0$ for all $x<\epsilon$. Without loss of generality, assume $\epsilon<a_2$. We now apply a shift transformation. Recall that $T_{-\epsilon}(\BFa):=(a_1-\epsilon,\ldots,a_s-\epsilon)$. Theorem \ref{theo:shift_one_sided} implies that
$\phi_{T_\epsilon(\BFa)}(x)=\sum_{j=1}^s A_j K(x+\epsilon,a_j). $
Thus,
\begin{eqnarray}
\tilde{\phi}_{T_\epsilon(\BFa)}(z)\propto \left [ \sum_{j=1}^{s} A_j \exp\{i (a_j-\epsilon) z\} \right](c^2+z^2)^{-(k-1)/2}.\nonumber
\end{eqnarray}
It is easily seen that $\tilde{\phi}_{T_\epsilon(\BFa)}(z)$ is unbounded if $z=iy$ for sufficiently large $y>0$. Specifically,
\begin{eqnarray*}
\tilde{\phi}_{T_\epsilon(\BFa)}(iy)&\propto& \left [ \sum_{j=1}^{s} A_j \exp\{- (a_j-\epsilon) y\} \right](c^2-y^2)^{-(k-1)/2}\\
&=& A_1 (c^2-y^2)^{-\frac{k-1}{2}} \exp\{(\epsilon-a_1)y\}+(c^2-y^2)^{-\frac{k-1}{2}}\sum_{j=2}^{s} A_j \exp\{- (a_j-\epsilon) y\},
\end{eqnarray*}
where the first term diverges and the second term converges to zero as $y\rightarrow+\infty$, because $A_1\neq 0$ and $a_1<\epsilon<a_2<\cdots<a_s$.

The remainder is to prove that $\operatorname{supp}\phi_{\BFa}=[0,+\infty)$. Suppose $\operatorname{supp}\phi_{\BFa}\neq [0,+\infty)$. Then there exist $M_2>M_1> 0$, such that $\phi_{\BFa}(x)=0$ whenever $x\in(M_1,M_2)$. Without loss of generality, we assume that $M_1,M_2\not\in \{a_1,\ldots,a_s\}$. Then we can write
$$\phi_{\BFa}(x)=\sum_{j=1}^s A_j K(x,a_j) +0\cdot K(x,M_1)+0\cdot K(x,M_2).$$
Then Lemma \ref{lem:connected} implies that $\phi_{\BFa}(x)$ can be decomposed into
$$\phi_{\BFa}(x)=\sum_{j=1}^\tau A_j K(x,a_j)+ \sum_{j=\tau+1}^s A_j K(x,a_j):=\psi_1(x)+\psi_2(x),$$
for some $1\leq \tau<s$, such that $\operatorname{supp}\psi_1\subset[0,M_1]$ and $\operatorname{supp}\psi_1\subset[M_2,+\infty)$. Because $0\in \operatorname{supp}\phi_{\BFa}$, we have $\operatorname{supp}\psi_1\neq \emptyset$. Therefore, $\psi_1$ is a non-zero function and has a compact support, which contradicts Theorem \ref{theo:matern_KPdegree}.
\end{proof}

\subsection{Linear Independence}

\begin{proof}[Proof of Theorem \ref{theo:linear_indpendent}] We have learned from Theorems \ref{theo:support1} and \ref{theo:support1_one_sided}, and the analogous counterpart of Theorem \ref{theo:support1_one_sided} for the left-sided KPs that:
\begin{enumerate}
    \item The left-sided KPs $\phi_1,\phi_2,\ldots,\phi_{(k-1)/2}$ have supports ${(-\infty,x_{(k+1)/2}]}$, ${(-\infty,x_{(k+1)/2+1}]}$, $\ldots$, ${(-\infty,x_{k-1}]}$, respectively.
    \item  The KPs  $\phi_{(k+1)/2},\phi_{(k+1)/2+1},\ldots,\phi_{n-(k-1)/2}$ have supports ${[x_1,x_{k}]}$, ${[x_2,x_{k+1}]}$, $\ldots$, ${[x_{n-k+1},x_n]}$, respectively.
    \item  The right-sided KPs  $\phi_{n-(k-3)/2},\ldots,\phi_{n-1},\phi_n$ have supports ${[x_{n-k+2},\infty)}$, $\ldots$, ${[x_{n-(k-1)/2-1},\infty)}$, ${[x_{n-(k-1)/2},\infty)}$ respectively.
\end{enumerate}
Therefore, for $\tau< n-(k-1)/2$, any function of the form $f:=\sum_{j=1}^\tau \lambda_j \phi_j$ satisfies $\operatorname{supp}f\subset (-\infty,x_{\tau+(k-1)/2}]$. Note that $\operatorname{supp}\phi_{\tau+1}=[x_{\tau+1-(k-1)/2},x_{\tau+1+(k-1)/2}]\not\subset (-\infty,x_{\tau+(k-1)/2}]$, which proves that $\phi_{\tau+1}\not\in \operatorname{span}\{\phi_1,\ldots,\phi_\tau\}$. Hence, by induction, we can prove that $\phi_1,\ldots,\phi_{n-(k-1)/2}$ are linearly independent.

Similarly, we can prove that the right-sided KPs $\phi_{n-(k-3)/2},\ldots,\phi_n$ are linearly independent.

Now suppose $\sum_{j=1}^n \xi_j \phi_j=0$ for $\xi_1,\ldots,\xi_n\in\mathbb{R}$. We rearrange this identity as
\begin{eqnarray}\label{linearlyindependent}
f_1:=\sum_{j=1}^{n-(k-1)/2} \xi_j \phi_j=-\sum_{j=n-(k-3)/2}^n \xi_j \phi_j=:f_2,
\end{eqnarray}
i.e., the left-hand side of (\ref{linearlyindependent}) is a linear combination of the left-sided KPs and the KPs, and the right-hand side of (\ref{linearlyindependent}) is a linear combination of the right-sided KPs.

Note that $\operatorname{supp}f_1\subset(-\infty,x_n]$ and $\operatorname{supp}f_2\subset[x_{n-k+2},+\infty)$. Then identify (\ref{linearlyindependent}) implies $\operatorname{supp}f_2\subset(-\infty,x_n]\cap[x_{n-k+2},+\infty)=[x_{n-k+2},x_n]$. 
By definition, $f_2$ is a linear combination of $k-1$ functions $K(\cdot,a_{n-k+2}),\ldots, K(\cdot,a_n)$. Hence, by Theorem \ref{theo:matern_KPdegree}, $f_2$ has a compact support only if $f_2\equiv 0$, which, together with the fact that $\phi_{n-(k-3)/2},\ldots,\phi_n$ are linearly independent, yields that $\xi_{n-(k-3)/2}=\cdots=\xi_n=0$. Then by (\ref{linearlyindependent}), we similarly have $\xi_1,\ldots,x_{n-(k-1)/2}=0$ because $\phi_1,\ldots,\phi_{n-(k-1)/2}$ are proved to be linearly independent. In summary, we prove that $\phi_1,\ldots,\phi_n$ are linearly independent.
\end{proof}

\bibstyle{apalike}
\vskip 0.2in
\bibliography{references}

\end{document}